\documentclass[a4paper]{article}

\usepackage[utf8]{inputenc} \usepackage{lmodern}
\usepackage[T1]{fontenc}    \usepackage{hyperref}       \hypersetup{
        colorlinks   = true,
        urlcolor     = blue,
        linkcolor    = blue,
        citecolor   = blue
      }
\usepackage{url}            \usepackage{booktabs}       \usepackage{amsfonts}       \usepackage{nicefrac}       \usepackage{microtype}      \usepackage{xcolor}         \usepackage[shortlabels]{enumitem}

\usepackage{amsmath,amsfonts,amssymb,amsthm,thmtools}
\usepackage[capitalize]{cleveref}
\usepackage[nolist,nohyperlinks]{acronym}
\usepackage{algorithm}
\usepackage{algpseudocode}
\usepackage{multirow}
\usepackage{array}
\usepackage{subcaption}
\usepackage{xspace}

\usepackage{dsfont}

\usepackage{mathtools}

\usepackage{tabularx}

\usepackage{fullpage}

\usepackage{natbib}
\setcitestyle{round,sectionbib}

\usepackage[colorinlistoftodos, shadow,color=blue!30!white
,disable
]{todonotes}

\usepackage{authblk}

\title{Pointwise confidence estimation in the non-linear $\ell^2$-regularized least squares}

\author[1]{Ilja Kuzborskij}
\author[2]{Yasin Abbasi Yadkori}
\affil[1]{Google DeepMind}
\affil[2]{Sapient Intelligence}
\date{} 
\def\ddefloop#1{\ifx\ddefloop#1\else\ddef{#1}\expandafter\ddefloop\fi}
\def\ddef#1{\expandafter\def\csname c#1\endcsname{\ensuremath{\mathcal{#1}}}}
\ddefloop ABCDEFGHIJKLMNOPQRSTUVWXYZ\ddefloop
\def\ddef#1{\expandafter\def\csname b#1\endcsname{\ensuremath{{\boldsymbol{#1}}}}}
\ddefloop ABCDEFGHIJKLMNOPQRSUVWXYZabcdeghijklmnopqrtsuvwxyz\ddefloop
\def\ddef#1{\expandafter\def\csname h#1\endcsname{\ensuremath{\hat{#1}}}}
\ddefloop ABCDEFGHIJKLMNOPQRSTUVWXYZabcdefghijklmnopqrsuvwxyz\ddefloop \def\ddef#1{\expandafter\def\csname hc#1\endcsname{\ensuremath{\widehat{\mathcal{#1}}}}}
\ddefloop ABCDEFGHIJKLMNOPQRSTUVWXYZ\ddefloop
\def\ddef#1{\expandafter\def\csname bar#1\endcsname{\ensuremath{\bar{#1}}}}
\ddefloop ABCDEFGHIJKLMNOPQRSTUVWXYZabcdefghijklmnopqrtsuvwxyz\ddefloop
\def\ddef#1{\expandafter\def\csname wbar#1\endcsname{\ensuremath{\overline{#1}}}}
\ddefloop ABCDEFGHIJKLMNOPQRSTUVWXYZabcdefghijklmnopqrtsuvwxyz\ddefloop
\def\ddef#1{\expandafter\def\csname tc#1\endcsname{\ensuremath{\widetilde{\mathcal{#1}}}}}
\ddefloop ABCDEFGHIJKLMNOPQRSTUVWXYZ\ddefloop

\DeclareMathOperator*{\argmin}{arg\,min}

\DeclareMathOperator{\E}{\mathbb{E}}

\DeclareMathOperator{\Var}{\mathrm{Var}}
\newcommand{\ind}{\mathbf{1}}

\renewcommand{\P}{\mathbb{P}}
\newcommand{\R}{\mathbb{R}}

\newcommand{\tp}{^{\top}}
\newcommand{\ip}[1]{\left\langle #1 \right\rangle}
\newcommand{\bmid}{\;\middle|\;}
\newcommand*\diff{\mathop{}\!\mathrm{d}}

\newcommand{\pr}[1]{\left( #1 \right)}
\renewcommand{\br}[1]{\left[ #1 \right]}
\newcommand{\cbr}[1]{\left\{ #1 \right\}}
\newcommand{\abs}[1]{\left|#1\right|}

\newcommand{\ve}{\varepsilon}
\renewcommand{\th}{\theta}

\newcommand{\repi}{^{(i)}}

\newcommand{\deli}{^{\backslash i}}

\newcommand{\sumin}{\sum_{i=1}^n}

\newcommand{\diag}{\mathrm{diag}}
\renewcommand{\vec}{\mathrm{vec}}
\renewcommand{\d}{\diff}
\newcommand{\op}{\mathrm{op}}

\newcommand{\xtest}{x}

\newcommand{\Lip}{\mathrm{Lip}}
\newcommand{\lmin}{\lambda_{\min}}
 
\newtheorem{lemma}{Lemma}

\newtheorem{theorem}{Theorem}

\newtheorem{proposition}{Proposition}
\newtheorem{definition}{Definition}
\newtheoremstyle{named}{}{}{\itshape}{}{\bfseries}{}{.5em}{\thmnote{#3.}#1}
\theoremstyle{named}

\begin{acronym}
\acro{RLS}{Regularized Least Squares}
\acro{ERM}{Empirical Risk Minimization}
\acro{RKHS}{Reproducing kernel Hilbert space}
\acro{DA}{Domain Adaptation}
\acro{PSD}{Positive Semi-Definite}
\acro{SGD}{Stochastic Gradient Descent}
\acro{OGD}{Online Gradient Descent}
\acro{GD}{Gradient Descent}
\acro{SGLD}{Stochastic Gradient Langevin Dynamics}
\acro{IS}{Importance Sampling}
\acro{WIS}{Weighted Importance Sampling}
\acro{MGF}{Moment-Generating Function}
\acro{ES}{Efron-Stein}
\acro{ESS}{Effective Sample Size}
\acro{KL}{Kullback-Liebler}
\acro{SVD}{Singular Value Decomposition}
\acro{PL}{Polyak-Łojasiewicz}
\acro{NTK}{Neural Tangent Kernel}
\acro{KLS}{Kernelized Least-Squares}
\acro{KRLS}{Kernelized Regularized Least-Squares}
\acro{ReLU}{Rectified Linear Unit}
\acro{NTRF}{Neural Tangent Random Feature}
\acro{NTF}{Neural Tangent Feature}
\acro{RF}{Random Feature}
\acro{KT}{Krichevsky–Trofimov}
\acro{ReLU}{Rectified Linear Unit}
\acro{CG}{Conjugate Gradient method}
\acro{MLE}{Maximum Likelihood Estimator}
\end{acronym} 

\begin{document}

\maketitle

\begin{abstract}
  We consider a high-probability non-asymptotic confidence estimation in the $\ell^2$-regularized non-linear least-squares setting with fixed design.
  In particular, we study confidence estimation for local minimizers of the regularized training loss.
  We show a \emph{pointwise} confidence bound, meaning that it holds for the prediction on any given fixed test input $\xtest$.
  Importantly, the proposed confidence bound
  scales with similarity of the test input to the training data in the implicit feature space of the predictor (for instance, becoming very large when the test input lies far outside of the training data).
  This desirable last feature is captured by the weighted norm involving the inverse-Hessian matrix of the objective function, which is a generalized version of its counterpart in the linear setting,
  $\xtest\tp \text{Cov}^{-1} \xtest$. 
  Our generalized result can be regarded as a non-asymptotic counterpart of the classical confidence interval based on asymptotic normality of the MLE estimator.
  We propose an efficient method for computing the weighted norm, which only mildly exceeds the cost of a gradient computation of the loss function.
  Finally, we complement our analysis with empirical evidence showing that the proposed confidence bound provides better coverage/width trade-off compared to a confidence estimation by bootstrapping, which is a gold-standard method in many applications involving non-linear predictors such as neural networks.
\end{abstract}

\section{Introduction}
\label{sec:intro}

One of the fundamental problems in statistics and machine learning is uncertainty quantification for some unknown ground truth function, such as constructing a high-probability confidence interval for a \emph{fixed} function $f(\xtest; \th^{\star})$ that is parameterized by a vector $\th^{\star}$ and evaluated at a \emph{fixed} point $\xtest$, when all we have is a collection of noisy observations $f(x_i; \th^{\star}) + \mathrm{noise}_i$ at \emph{fixed} training inputs $x_1, \ldots, x_n$. Given an uncertainty estimate, the practitioner might decide to collect additional data in regions of high uncertainty or abstain from using the model if the model is deemed too unreliable.

The construction of confidence intervals (or, more generally, confidence sets) is a long-lived topic of interest at the heart of statistics ~\citep{van2000asymptotic,shao2003mathematical,wasserman2004all}.
At the same time, confidence intervals fall in contrast with the de facto approach in machine learning literature, where one obtains a \emph{single} estimate of $f(\xtest; \th^{\star})$  based on the loss minimization~\citep{shalev2014understanding}.
While it is easy to estimate the out-of-sample performance of such standard learning approaches under the i.i.d.\ assumption (say, using a held-out sample), the standard approach often falls short of identifying the limits of our knowledge about $f(\xtest; \th^{\star})$ for an \emph{arbitrary} input point $\xtest$. This problem is particularly relevant when the input point $\xtest$ is far from the training data.

A notable success is a linear regression setting where the ground truth is given by $f(x; \th^{\star}) = x\tp \th^{\star}$, and where the solution to a (regularized) linear least-squares problem $\hat \th$ can be used to state a \emph{high-probability} confidence interval of a form
$|{\xtest}\tp \hat \th - {\xtest}\tp \th^{\star}|^2 \lesssim {\xtest}\tp \Sigma^{-1} \xtest / n$
where $\Sigma$ is a covariance matrix.
Here an important term is the
\emph{weighted norm} ${\xtest}\tp \Sigma^{-1} {\xtest}$ that captures how well the test input $\xtest$ is
represented in the span of training inputs, and so if $\xtest$ is very dissimilar from the training data, the norm will be large.
Similar confidence intervals can also be stated in the non-linear case by lifting derivations to the formalism of kernel methods~\citep{RasmussenW06,Srinivas2010,Yadkori2012,chowdhury2017kernelized}.

In this paper we focus on uncertainty estimation of \emph{non-linear} ground truth functions which can be modelled by highly non-linear predictors, such as overparameterized neural networks.
While a large body of recent literature has explored the practical adaptation of ideas from statistics to neural networks, such as bootstrap and ensemble methods~\citep{dwaracherla2022ensembles}, these tend to require strong assumptions on the distribution of the test input point $\xtest$ -- usually that it is an i.i.d. sample from the same distribution as the training set.
Another limitation of these approaches is computational:
to control their variance, a large number of dataset replications and re-training steps are required, which is impractical in the context of such predictors.

In contrast to these approaches, in this paper we are motivated by the presence of the weighted norm beyond the linear case.
Intuitively, one would expect that in the non-linear case the weighted norm should bear a similar meaning as in the linear case, but now capturing whether the test point resides in the span of some implicit `features' learned by the algorithm.
Moreover, we ask whether one can attain high-probability non-asymptotic confidence intervals that would be shaped by this kind of weighted norm. To this end, the concept of weighted norms in this context is not new.
A fundamental result in parametric estimation going back to works of R.~Fisher is the asymptotic normality of the \ac{MLE} $\hat \th$, namely
$\sqrt{n} \, (\hat \th - \th^{\star}) \overset{d}{\longrightarrow} \cN(0, \cI(\th^{\star})^{-1})$,
where the
ground truth is parameterized by a fixed vector $\th^{\star}$
and
where $\cI(\th^{\star})$ is a \emph{Fisher information matrix} \citep{van2000asymptotic}.
This result is used to give various \emph{asymptotic} confidence sets for $\th^{\star}$, such as Wald-type elliptic sets and a so-called linearized Laplace approximation in the context of Bayesian literature~\citep{Mackay1992Thesis,antoran2022adapting} (see \Cref{sec:related} for a detailed discussion of related work).
Constructing 
confidence intervals on $f(\xtest; \th^{\star})$ based on asymptotic normality would indeed involve a generalized weighted norm defined with respect to the Fisher information matrix.
However, such methods occasionally suffer from coverage issues due to their asymptotic nature and a lack of high-probability guarantee, especially when the sample is small.
In this paper we explore this problem from the viewpoint of concentration inequalities
by focusing on the \emph{point variance} of the estimator:
\begin{align}
  \label{eq:ci-decomp}
  |f(\xtest; \hat \th) - f(\xtest; \th^{\star})|
  \leq
  \underbrace{
  |f(\xtest; \hat \th) - \E[f(\xtest; \hat \th)]|
  }_{\text{\color{blue} Point variance}}
  +
  \underbrace{|\E[f(\xtest; \hat \th)] - f(\xtest; \th^{\star})|}_{\text{Bias}}~.
\end{align}
Note that the point variance is a \emph{random} part of the confidence interval, while
we do not consider here a \emph{deterministic} \emph{bias} of the estimator, controlling which is typically a hard problem on its own requiring assumptions about the regularity of the ground truth function~\citep{gyorfi2006distribution}.
\paragraph{Our contributions}
In this work we show a high-probability non-asymptotic confidence bound that holds generally in the $\ell^2$-regularized non-linear least-squares setting with fixed design (\Cref{thm:ci-hp-ls}).
The only requirement that we impose on the solution $\hat \th$ is to be a local minimizer of this (differentiable) problem.
Note that this is a very general setting where $f$ might be a highly non-linear predictor, such as a neural network. %
In particular, we show that with high probability
\begin{align*}
  |f(\xtest; \hat \th) - \E[f(\xtest; \hat \th)]| = \cO_P\pr{
  \frac{\|\nabla_{\th} f(\xtest; \hat \th)\|_{M(\hat \th)}}{\sqrt{n}}
  + \frac{1}{n}
  }
  \qquad \text{as} \qquad n \to \infty
\end{align*}
where $M(\th)$ is a particular \ac{PSD} empirical weighting matrix that asymptotically converges to the Fisher information matrix.
In \Cref{sec:results} we study this bound in various settings and show that it recovers known results for the linear setting,
confidence bounds with weighted norms,
and that it is asymptotically correct, meaning that it matches the shape of the confidence interval based on asymptotic normality.

Next, we propose an efficient algorithm for computation of the weighted norm, which exploits automatic differentiation techniques to attain the worst-case computational complexity of order $\tilde \cO(p n)$ in contrast to a naive computation of order $\cO(p^3 + p n)$
where $p$ is the number of parameters and $n$ is the sample size.
Meanwhile, the memory complexity is of the order $\cO(p)$ instead of $\cO(p^2)$. 
In the interpolation regime, the computational cost is improved to $\max(p, C_L)$, where $C_L$ is the cost of evaluating the loss function. 

Finally, we evaluate our bound experimentally in \Cref{sec:experiments} and show that it is able to attain a good coverage / width trade-off on a non-linear problem even in the presence of regions of the input space without observations.
In comparison, we show that bootstrapping, as a representative often used go-to method, falls short in terms of coverage.

\section{Preliminaries}
\label{sec:preliminaries}
\paragraph{Definitions}
Throughout, $\|\cdot\|$ is understood as the Euclidean norm for vectors and the Frobenius norm for matrices.
For some vector $x \in \R^p$ and a \ac{PSD} matrix $A \in \R^{p \times p}$, the weighted Euclidean norm is defined as
$\|x\|_A^2 = x\tp A x$.
We will use a shorthand notation for gradient evaluation at a point, $\nabla_\th f(\dots, \th', \dots) = (\d f/\d \th) |_{\th=\th'}$, e.g., $\nabla_\th f(\dots, \hat \th, \dots)$, $\nabla_x f(\dots, x_i, \dots)$.

We will also work with unbounded random variables which are often captured through the following tail conditions (see also \citep[Section 2.4]{boucheron2013concentration}).
\begin{definition}[Sub-gamma random variable]
  A random variable is called sub-gamma on the right tail with variance factor $v^2$ and scale $c > 0$ if
  \begin{align*}
    \ln \E[e^{\lambda (X - \E[X])}] \leq \frac{\lambda^2 (v^2/2)}{1-c \lambda} \qquad \text{for every} \quad 0 < \lambda < 1/c~.    
  \end{align*}
  Similarly, $X$ is called sub-gamma on the left tail if $-X$ is sub-gamma on the right tail with parameters $(v^2, c)$.
  Finally, $X$ is called sub-gamma if conditions on both tails hold simultaneously.
\end{definition}
One intuitive consequence of a sub-gamma condition is that it implies tail bounds
$\P\{X - \E[X] > \sqrt{2 v^2 t} + c t\} \leq e^{-t}$ and $\P\{\E[X] - X > \sqrt{2 v^2 t} + c t\} \leq e^{-t}$,
which tells us that depending on its parameters, it might concentrate around its mean at rates ranging from
$\sqrt{2 v^2 t}$ to  $c t$. Note that when $|X|$ is bounded by $B$, we have $v^2 = B^2/4$ and $c=0$.
Sub-gamma condition on the tail is weaker than sub-gaussianity, but stronger than assuming that the variance of a random variable is bounded.

\subsection{Setting}
In the following we will look at the parameterized predictors $f : \R^d \times \R^p \to \R$;
while $f(x; \th)$ will stand for evaluation on an input $x \in \R^d$ given parameters $\th \in \R^p$.
In practice, parameters $\hat \th$ are found given training examples $(x_1, Y_1), \ldots, (x_n, Y_n) \in \R^d \times \R$.
In this paper, we work in a \emph{fixed-design regression setting}
where inputs $x_1, \ldots, x_n$ are fixed, while responses $Y_1, \ldots, Y_n$ are random.
In particular, each response is generated as
\begin{align}
  \label{eq:regression}
  Y_i = f(x_i; \th^{\star}) + Z_i~, \qquad Z_i \sim \cN(0, \sigma^2)~.
\end{align}
where $f(\cdot \, ; \th^{\star}) : \R^d \to \R$ is some unknown fixed regression function.
Then, the \emph{empirical risk} of $\th$ with respect to the square loss and its $\ell^2$-regularized counterpart are respectively defined as
\begin{align*}
  L(\th) = \frac{1}{2 n} \sumin (f(x_i; \th) - Y_i)^2~,
  \qquad
  L_{\lambda}(\th) = L(\th) + \frac{\lambda}{2} \|\th\|^2~.
\end{align*}
Given a \emph{fixed} test point $\xtest$ our goal in this paper is to estimate the \emph{point variance},
which we will see as a proxy to uncertainty within the model arising due to randomness in responses,
and which is an
important part in a general pointwise uncertainty estimation problem, \cref{eq:ci-decomp}.
In this paper, we do not control a (deterministic) bias of a learning algorithm, unless in some sanity-check
scenarios (such as when $f(\xtest; \th^{\star})$ is linear).
Controlling the bias typically requires smoothness assumptions on the regression function $f(\xtest; \th^{\star})$ (such as membership in a certain Sobolev space), which in turn requires to show that the algorithm is able to learn well in very large (non-parametric) classes of functions.
As a consequence, such bounds are often quite pessimistic and involve a form of the `curse-of-dimensionality'~\citep{gyorfi2006distribution}.
Many gold-standard confidence estimation methods, such as the ones derived from bootstrap~\citep{efron1994introduction}, forego bias control.

In particular, in this paper we will study point variance estimation using $x \mapsto f(x \,; \hat \th)$, where
parameter vector $\hat \th$ is assumed to be a local minimizer of $L_{\lambda}$.
Moreover, we require $L_{\lambda}$ to be differentiable in $\th$.
This implies that $\hat \th$ is a stationary point of $L_{\lambda}$, or in other words, $\nabla L_{\lambda}(\hat \th) = 0$.
In turn, this implies an important property of \emph{alignment} beween the parameter vector and the gradient:
\begin{align}
  \label{lem:alignment}
  \lambda \hat \th = - \nabla L(\hat \th)~.
\end{align}
Note that the stationarity is asymptotically achieved (and so is alignment) under appropriate smoothness and boundedness conditions by many optimization algorithms used in practice, such as \ac{GD} (and its stochastic counterpart) \citep{GhadimiL13}, as well as adaptive ones such as Adam/Adagrad~\citep{defossez2022simple,li2023convergence}, and AdamW~\citep{zhou2024towards}.
As an example, we show the alignment property for \ac{GD} in \Cref{sec:app:preliminaries}.
\section{Main results}
\label{sec:results}
Our first result is a high-probability concentration inequality on the
point variance gap
\begin{align*}
  \Delta(\xtest; \hat \th) = f(\xtest; \hat \th) - \E[f(\xtest; \hat \th)]~.
\end{align*}
Here we assume that $\hat \th$ is any local minimizer of $L_{\lambda}$, and we require the bound on the gap to be given
in terms of a data-dependent quantity that captures sensitivity of the predictor to new observations.
In the following, we will often refer to this quantity as the \emph{weighted norm} with respect to $(\hat \th, \xtest)$, and which is given by $\|\nabla_\th f(\xtest; \hat \th)\|^2_{M(\hat \th)}$ where we define a \ac{PSD} matrix
\begin{align*}
  M(\th) = \pr{\nabla^2 L(\th) + \lambda I}^{-1} \pr{\frac1n \sumin \nabla_\th f(x_i; \th) \nabla_\th f(x_i; \th)\tp} \pr{\nabla^2 L(\th) + \lambda I}^{-1}~.
\end{align*}
Note that the weighted norm is random only in the responses $Y_1, \ldots, Y_n$, and moreover it can potentially be unbounded.
To capture the fact that it is unbounded in our analysis, it is natural to assume standard tail conditions for such a scenario.
In particular, here we will assume that the weighted norm is a \emph{sub-gamma} random variable with variance factor $v^2$ and scale $c > 0$, which captures many cases of bounded and unbounded random variables (see \Cref{sec:preliminaries} for definition and discussion).

Now we present our main result, which shows that the point variance gap is controlled in terms of the weighted norm (the proof is given in \Cref{sec:proofs}):\footnote{
  One can also show a weaker result (with dependence $1/\sqrt{\delta}$ rather than $\ln(1/\delta)$), which does not require any tail assumptions on the weighted norm at all (see \Cref{sec:poly-rates}). 
}
\begin{theorem}
  \label{thm:ci-hp-ls}
  Assume that $L_{\lambda}$ is differentiable and that $\hat \th$ is its local minimizer.
  Assume that the centered weighted norm with respect to $(\hat \th, \xtest)$
is a sub-gamma random variable with variance factor $v^2$ and scale $c$.
    Then, for any fixed test input $\xtest$ and for any $\delta \in (0,1)$ \begin{align*}
      \P\cbr{\abs{\Delta(\xtest; \hat \th)} >
        \sigma \sqrt{\frac{(\pi^2/2) \ln\pr{2/\delta}}{n} \E\br{\|\nabla_\th f(\xtest; \hat \th)\|^2_{M(\hat \th)}} }
      +
        \sigma^2 \, \frac{\sqrt{2 \ln\pr{2/\delta}} \, v + (2/3) \ln\pr{2/\delta} c}{n}
      } \leq \delta~.
    \end{align*}
  \end{theorem}

  Note that an empirical bound can be stated from the above as a consequence of the sub-gamma assumption, where the expected weighted norm is replaced by
$\|\nabla_{\th} f(\xtest; \hat \th)\|_{M(\hat \th)}^2 + \sqrt{2 v^2 \ln(2/\delta)} + c \ln(2/\delta)$.
Next, we discuss some implications and analyze \Cref{thm:ci-hp-ls} in various settings.
\paragraph{Sanity check for linear least-squares}
First, as a sanity check, consider a linear model where $f(\xtest; \th^{\star}) = x\tp \theta^\star$ for some fixed ground-truth parameter vector $\theta^\star$.
Then $f(x; \th) = x\tp \th$ and we employ a well-known ridge regression estimator
\begin{align*}
  \hat \th = \Sigma^{-1}\pr{\frac1n \sumin x_i Y_i} \qquad \text{where} \qquad \Sigma = \frac1n \sumin x_i x_i\tp + \lambda I~.
\end{align*}
Now, it is straightforward to check that (see \Cref{sec:linear-least-squares})
\begin{align*}
\E[f(\xtest; \hat \th)]
= \xtest\tp \th^{\star} - \lambda \, \xtest\tp \Sigma^{-1} \th^{\star}
\qquad\text{and}\qquad
  \E\br{\|\nabla_\th f(\xtest; \hat \th)\|^2_{M(\hat \th)}} \leq \|\xtest\|_{\Sigma^{-1}}^2
\end{align*}
and so we recover a confidence interval for such a case:
\begin{align*}
  \P\cbr{|\xtest\tp (\hat \th - \th^{\star})| > \lambda \abs{ \xtest\tp \Sigma^{-1} \th^{\star}}
  +
    \sigma \sqrt{\frac{(\pi^2/2) \ln\pr{2/\delta}}{n}} \, \|\xtest\|_{\Sigma^{-1}}
} \leq \delta~.
\end{align*}
This confidence interval, importantly, scales with the weighted norm $\|\xtest\|_{\Sigma^{-1}}$,
has a bias term $\lambda \abs{ \xtest\tp \Sigma^{-1} \th^{\star}}$,
and matches known high-probability confidence intervals for ridge regression (with slightly worse constant), see for instance~\citep[Chap. 20]{LattimoreS18}.
The dependence on this weighted norm is optimal in the fixed and adaptive design case~\citep{lattimore2023lower}.
\paragraph{Relationship of matrix $M(\th)$ to inverse of the Fisher information matrix}
\label{para:fisher}
Matrix $M(\hat \th)$ appearing in the weighting of $\|\nabla f(\xtest; \hat \th)\|_{M(\hat \th)}^2$, is closely related to the inverse of the \emph{Fisher information matrix} $\cI(\th)$ appearing in theory of parametric estimation (more on that in the next paragraph).
In the context of non-linear least-squares model considered here, $\cI(\th) = \E[\nabla^2 L(\th)]$
, whose inverse is somewhat different from $M(\th)$.
However, asymptotically under some conditions, they coincide.

A better case for our regression model is when $f$ predicts nearly as well as the regression function $f(x; \th^{\star})$, or generalizes well.
In this case, one can show that $\E M(\th^{\star})$ is essentially replaced by the inverse of $\cI(\th^{\star})$ assuming a particular data-dependent tuning of $\lambda \to 0$ (see \cref{prop:almost-fisher} in \Cref{sec:app:results}),
\begin{align*}
  \E\br{\|\nabla_{\th} f(\xtest; \th^{\star})\|_{M(\th^{\star})}^2}
  \leq
  \|\nabla_{\th} f(\xtest; \th^{\star})\|_{\cI(\th^{\star})^{-1}}^2 + C \cdot \frac{\sigma^2}{n}~.
\end{align*}
Note that typically we learn such a good predictor when $n \to \infty$ and as $\lambda \to 0$, and so the above dependence on $n$ and $\lambda$ becomes negligible. \paragraph{Asymptotic efficiency and connection to MLE}
\label{sec:connection-to-mle}
Now we look at another theme closely related to the Fisher information.
Let $\hat \th_n^{\textsc{mle}}$ be the \ac{MLE} of the true fixed parameter vector $\th^{\star}$.
It is well-known that \ac{MLE} satisfies asymptotic normality under several standard assumptions,
including \emph{consistency} in the sense that $\hat \th_n^{\textsc{mle}} \longrightarrow \th^{\star}$ in distribution as $n \to \infty$ \citep{shao2003mathematical}.
Then, in combination with the delta method, \begin{align*}
  \sqrt{n} \, \pr{f(\xtest; \hat \th_n^{\textsc{mle}}) - f(\xtest; \th^{\star})} \overset{d}{\longrightarrow} \cN\pr{0, \|\nabla_{\th} f(\xtest; \th^{\star})\|_{\cI(\th^{\star})^{-1}}^2}~.
\end{align*}

In comparison, \Cref{thm:ci-hp-ls} can be thought of as a non-asymptotic variant of this result adapted to non-linear least-squares.
To demonstrate this, consider the asymptotic bound implied by \Cref{thm:ci-hp-ls}.
Under some differentiability assumptions, and the convergence $\hat \th_{n} \overset{a.s.}{\longrightarrow} \hat \th_{\infty}$, we almost surely have
\begin{align*}
  \limsup_{n \to \infty} \frac{\sqrt{n}}{\ln(n)} \, \abs{f(\xtest; \hat \th_n) - \E[f(\xtest; \hat \th_n)]}
  \leq
  \sigma \pi \sqrt{\E \|\nabla_{\th} f(\xtest; \hat \th_{\infty})\|_{M(\hat \th_{\infty})}^2}~.
\end{align*}
This bound carries a similar message to the one we get from the asymptotic normality of \ac{MLE}: the prediction asymptotically concentrates around the expected prediction and the `strength' of the concentration is attenuated by the direction $\xtest$ (see \Cref{prop:asymptotics} in \Cref{sec:app:results} for the proof).

However, this statement is weaker since we do not speak about convergence to the ground truth $f(x; \th^{\star})$.
At this point, similarly to the case of the asymptotic normality of the \ac{MLE} we can assume that the algorithm is consistent in a sense that $\hat \th_{\infty} = \th^{\star}$.
Then, the above asymptotic result combined with an earlier observation that $\E M(\hat \th_{\infty})$ converges to the inverse of the Fisher information matrix,
\begin{align*}
  \limsup_{n \to \infty} \frac{\sqrt{n}}{\ln(n)} \, \abs{f(\xtest; \hat \th_n) - f(\xtest; \th^{\star})}
  \leq
  \sigma \pi \sqrt{\|\nabla_{\th} f(\xtest; \th^{\star})\|_{\cI(\th^{\star})^{-1}}^2}~.
\end{align*}
\paragraph{Uniform bounds over $\xtest$}
\Cref{thm:ci-hp-ls} holds only for a fixed $\xtest$, however in practice,
we might want to have a confidence bound that holds over all test inputs within
some set simultaneously, for example, over a unit sphere $\mathbb{S}^{d-1}$. One standard way to achieve this is through the covering number technique that we employ here (see \Cref{sec:uniform-bound} for details).
Let $\texttt{bound}(\ln(1/\delta))$ be the bound on $\Delta(\xtest; \hat \th)$ given in \Cref{thm:ci-hp-ls}, as a function of $\ln(1/\delta)$ term.
Then, it is straightfoward to show that (see \Cref{sec:uniform-bound})
\begin{align*}
  \P\cbr{
  \max_{x \in \mathbb{S}^{d-1}} |\Delta(x; \hat \th)| >
  \texttt{bound}\pr{d \, \ln\pr{\frac{3 n \Lip(\Delta)}{\delta}}}
  +
  \frac1n
  } \leq \delta
\end{align*}
where $\Lip(\Delta)$ is a Lipschitz constant of the gap $x \mapsto \Delta(x; \hat \th)$.
By examining this bound, note that we incur a factor $\sqrt{d}$ and the logarithm of the Lipschitz constant of the gap, compared to the case when $\xtest$ is fixed.
Note that the latter can be simply related to the Lipschitz constant of the predictor in the input,
which in practice can be controlled by the training procedure (say, by adding a constraint $\Lip(f(\cdot \,; \, \hat \th)) \leq B$).
On the other hand, in some cases, the Lipschitz constant is already implicitly controlled by the algorithm, which is often the consequence of $\ell^2$ regularization.
As an example, here we look at the case when $f$ is a multilayer network trained to stationarity of $L_{\lambda}$.

Consider a multilayer neural network \begin{align}
  \label{eq:fcn}
  f(x; \th) = w_K\tp h_{K-1}, \quad h_k = a(W_k h_{k-1}), \quad h_0 = x \qquad (k \in [K-1]),
\end{align}
where $x$ represents the input.
Here $W_1, \ldots, W_K$ are weight matrices collectively represented by the parameter vector $\th = (\vec(W_1), \ldots, \vec(W_K))$,
and $a : \R^m \to \R^m$ is the activation function such that $\max_k \sup_{x \in \R} |a_k'(x)| \leq 1$ (such as ReLU).

Then, assuming that $\hat \th$ is found by \ac{GD}, and the regularized loss at initialization is bounded by a constant $C$ (say, with high probability over data),
we can control the Lipschitz constant for such a predictor, see \Cref{lem:fcn-lip}.
Then, \Cref{thm:ci-hp-ls} implies
\begin{align*}
  \P\cbr{\max_{x \in \mathbb{S}^{d-1}} |\Delta(x; \hat \th)| >
  \texttt{bound}\pr{d \, \ln\pr{\frac{6 n }{\delta} } + \frac{K d}{2} \ln\pr{ \frac{C}{\lambda K}}}
  + \frac1n
  } \leq \delta
\end{align*}
where $\texttt{bound}(\ln(1/\delta))$ is the bound on $\Delta(\xtest; \hat \th)$ given in \Cref{thm:ci-hp-ls}, as a function of $\ln(1/\delta)$ term.
We pay a factor $\sqrt{K d \ln(1/\lambda)}$ in the final bound, which is a mild cost compared to the linear case where we would only incur $\sqrt{d}$, considering that the number of parameters can be much larger than $d$.

\paragraph{Relationship to the \ac{NTK}}
  \label{sec:ntk}
To compare our approach to the kernel setting, we first briefly sketch the proof idea of \Cref{thm:ci-hp-ls}.
Our approach can be summarized as follows: the first component is \Cref{lem:norm2norm},
whose key part states that \begin{equation}
\label{eq:temp1}
\frac{\partial \hat{\theta}}{\partial z_i} =  \frac{1}{n} (\lambda I + \nabla^2 L(\hat{\theta}))^{-1} \nabla_\theta f(x_i, \hat{\theta})
\end{equation}
where on the left we differentiate $\hat \th$ with respect to noise variable $z_i$.
Then, the chain rule implies that
\[
\|\nabla_{(z_1, \ldots, z_n)} f(\xtest; \hat \th)\|^2 = \frac{1}{n^2} \, \| \nabla_\theta f(x, \hat \theta)\|_{M(\hat \th)}^2
\]
and,
finally, \Cref{lem:F-concentration-subgamma} states a concentration inequality in terms of $\E[\|\nabla_{(z_1, \ldots, z_n)} f(\xtest; \hat \th)\|^2]$. 

To make the connections with kernel regression clear, notice that \cref{eq:temp1} implies a linear approximation
\begin{equation*}
\label{eq:lin-approx}
\hat\theta \approx \th' + \frac{1}{n}(\lambda I + \nabla^2 L(\hat{\theta}))^{-1} \sum_i Z_i \nabla_\theta f(x_i; \hat{\theta})\,,
\end{equation*}
where $\th'$ would be the solution if all noise terms were zero. With the choice of $\th'=(1/n)(\lambda I + \nabla^2 L(\hat{\theta}))^{-1} \sum_i \nabla_\theta f(x_i; \hat{\theta}) \nabla_\theta f(x_i; \hat{\theta})^\top \theta^{\star}$ the above approximation has the exact same form as the least-squares solution when $\nabla_\th f(x_i; \hat \th)$ does not depend on the noise terms (i.e.\ if we had magically initialized $\th$ with $\hat\th$), and the model is a linear map with weight vector $\theta^{\star}$. The \ac{NTK} is such a case, where the kernel is defined with feature maps that only depend on the initialization~\citep{jacot2018neural},
\[
K_0(x,x') = \frac{1}{n}\ip{\nabla_\theta f(x; \theta_0), \nabla_\theta f(x'; \theta_0)} \;.
\]
For \cref{eq:temp1} to hold, we need regularization (for the alignment condition of \cref{lem:alignment} to hold) and $\hat \th$ must be a local minimizer, and no linearity is needed. Therefore, our result is more general and applies to the stationary point of the gradient descent process, $\hat \th$, which is equivalent to a kernel regression with feature map $\nabla_\th f(x_i; \hat \th)$. To see this, define time-varying kernel 
\[
K_t(x,x') = \frac{1}{n}\ip{\nabla_\theta f(x; \theta_t), \nabla_\theta f(x'; \theta_t)} \;.
\]
The update in the function space after a small gradient update with step size $\eta$ can be written as
\begin{align*}
f(x; \theta_{t+1}) &\approx  f(x; \theta_{t}) + \ip{\nabla_\theta f(x; \theta_t), \theta_{t+1} - \theta_t}\\
&= f(x; \theta_{t}) - \ip{\nabla_\theta f(x; \theta_t), \frac{\eta}{n} \sum_i (f(x_i; \theta_t) - y_i)\nabla_\theta f(x_i; \theta_t) } \\
&= f(x; \theta_{t}) - \eta \sum_i K_t(x, x_i) (f(x_i; \theta_t) - Y_i) \;.
\end{align*} 
Our analysis is also more general. For obtaining uniform bounds in the linear case, the typical approach is to use the Cauchy-Schwarz inequality to separate the weighted norm from the noise, and then show a self-normalized bound for the noise term. This approach is heavily based on linearity. Our approach, on the other hand, requires the parameter solution to be differentiable with respect to noise, and is not applicable to, e.g., Bernoulli noise. Also, the weighted norm on the r.h.s.\ in the linear case does not involve an expectation and no sub-gamma assumption is needed. %
A less elegant alternative to our approach would be to split the data, train the network with the first half, learn the ``features'', and then apply kernel regression on the second half and construct confidence intervals. A common data-splitting approach in practice is to learn the feature maps up to the last layer, and then use extra data to construct confidence ellipsoids around the weight matrix of the last linear layer.

\subsection{Efficient computation of a weighted norm}
\label{sec:alg}
The calculation of the weighted norm involves the matrix $M(\hat \th)$ which in turn involves an $p \times p$ inverse Hessian matrix of the loss function $L_{\lambda}$, which presents significant computational challenges since $p$ is often large.
Explicitly forming the Hessian has a computational cost $\cO(p^2 \cdot C_L)$ (where $C_L$ is the cost of evaluating $L_{\lambda}$), storing it requires $\cO(p^2)$ memory, and inverting it takes $\cO(p^3)$ operations, all of which are often infeasible.
To this end, we consider a well-known and efficient alternative that leverages the \ac{CG} which can be used to iteratively solve the linear system $H h = v$ for $h=H^{-1}v$~\citep{pearlmutter1994fast}. \ac{CG} avoids forming $H$, requiring only the computation of Hessian-vector products (HVPs). An HVP can be computed efficiently using automatic differentiation without forming $H$ explicitly, and is easily supported by modern packages such as JAX and PyTorch.
These packages compute an HVP at a cost proportional to that of a single gradient evaluation,
thereby circumventing the $\cO(p^2)$ bottleneck.
While these techniques are relatively well known, we adapt them to our case and include pseudocode for completeness in \Cref{sec:app:alg}.

In our case, for \ac{CG} to converge we need
$k = \Omega(\sqrt{1 + B^2/\lambda} \ln(1/\ve))$ iterations, where $B = \max_i \|\nabla_\th f(x_i; \hat \th)\|$ and $\ve$ is the desired precision.
We include a pseudocode for the Hessian-inverse-vector product in \Cref{alg:hinvp} in \Cref{sec:app:alg}, while its computational complexity is of the order
\begin{align*}
  \cO\pr{(C_L + p) \sqrt{1 + \frac{B^2}{\lambda}} \ln\pr{\frac{1}{\ve}}}
\end{align*}
  which is a stark difference compared to $\cO(p^3)$.

Finally, the main algorithm that computes the weighted norm (\Cref{alg:weighted-norm}) calls a Hessian-inverse-vector product operation once. Then, the total computational complexity of the main algorithm in the worst case is
$$\tilde{\cO}\pr{\frac{(n + p)}{\sqrt{\lambda}} +  n p}$$  assuming that the cost of loss evaluation is $C_L = \cO(n)$.
Note that in the interpolation regime when $\lambda = 0$, the computational complexity of \Cref{alg:weighted-norm} is only $\tilde{\cO}((C_L + p)/\sqrt{\lmin} +  p)$ where $\lmin$ is the smallest positive eigenvalue of the Hessian matrix.

After examining the proof of \Cref{thm:ci-hp-ls} (see \Cref{sec:proofs} or \cref{sec:ntk} for the sketch),
one can argue that it might be easier to compute $\|\nabla_{(z_1, \ldots, z_n)} f(\xtest; \hat \th)\|^2$ which is identical to the weighted norm.
While $\nabla_{z_i} f(\xtest; \hat \th) = \nabla_\theta f(x, \hat \theta)^\top (\partial \hat{\theta} / \partial z_i)$, we might be tempted to use automatic differentiation to directly compute $(\partial \hat{\theta} / \partial z_i)$.
If we use plain \ac{GD} and terminate the gradient updates after $\tau$ rounds (i.e. approximate $\hat\th=\theta_{\tau}$),
this is indeed possible by taking the chain rule: $(\partial \theta_\tau / \partial z_i) = (\partial \theta_{\tau} / \partial \theta_{\tau-1}) \dots (\partial \theta_{1} / \partial z_i)$ and boils down to $\tau$ Jacobian-vector product operations.
This naive approach would require $\cO(\tau \, (C_L + p))$ computation and $\cO(\tau \, C_L)$ memory.
Thus, if we \emph{save all} vectors $(\partial \hat{\theta} / \partial z_i)_{i=1}^n$ in memory, we can calculate the desired $\|\nabla_{(z_1, \ldots, z_n)} f(\xtest; \hat \th)\|^2$ term in $\cO(n p)$ time and memory.
Instead, if we employ an adaptive algorithm such as Adam, this would require to store \emph{the entire} parameter vector trajectory.
Moreover, this method is not applicable when all we have is access to the final parameter vector but not to the intermediate iterations.
\section{Proof of \Cref{thm:ci-hp-ls}}
\label{sec:proofs}
Throughout proofs in this section we will occasionally use abbreviations:
\begin{itemize}
\item $\hat \th\repi(z)$ to indicate parameter obtained on the sample where $Z_i$ is replaced by $z$
\item $f(\th) = f(\xtest; \th)$
\end{itemize}
In this section we first show a general result which holds beyond the regression model of \cref{eq:regression} and the square loss.
The following \Cref{thm:ci-hp} holds for any differentiable loss function as long as it is differentiable in the noise variable.
For the sake of generality, we will introduce loss functions $\ell_i : \R^p \times \cZ \to \R$ for $i \in [n]$ where $\cZ$ is a real field.
Then, slightly abusing the notation, the empirical risk is now defined as
\begin{align*}
  L(\th) = \frac1n \sumin \ell_i(\th, Z_i)~.
\end{align*}
\begin{theorem}
  \label{thm:ci-hp}
  Assume that $L_{\lambda}$ is differentiable and that $\hat \th$ is its local minimizer.
  Assume that $\|\nabla_\th f(\xtest; \hat \th)\|^2_{M(\hat \th)} - \E[\|\nabla_\th f(\xtest; \hat \th)\|^2_{M(\hat \th)}]$ is a sub-gamma random variable with variance factor $v^2$ and scale $c$.
  Then, for any fixed test input $\xtest$ and for any $\delta \in (0,1)$,
    \begin{align*}
      \P\cbr{
\abs{\Delta(\xtest; \hat \th)} > 
      \sqrt{\frac{(\pi^2/2) \ln(2/\delta)}{n} \E\br{\|\nabla_\th f(\xtest; \hat \th)\|^2_{M(\hat \th)}} }
      +
      \frac{\sqrt{2 \ln(2/\delta)} \, v + (2/3) \ln(2/\delta) c}{n}
      } \leq \delta~.
    \end{align*}
  where
  \begin{align*}
    M(\hat \th) = \pr{\nabla^2 L(\hat \th) + \lambda I}^{-1} \pr{\frac1n \sumin G_i G_i\tp} \pr{\nabla^2 L(\hat \th) + \lambda I}^{-1} \,\,\, \text{and} \,\,\, G_i = \frac{\d}{\d z} \nabla_\th \ell_i(\hat \th, z) \bigg|_{z=Z_i}~.
  \end{align*}
\end{theorem}
Then, \Cref{thm:ci-hp-ls} is a consequence of the above.
Namely, specializing \Cref{thm:ci-hp} to $\ell_i(\th, Z_i) = (1/2) (f(x_i; \th) - Y_i)^2 = (1/2) (f(x_i; \th) - f^{\star}(x_i) - \sigma Z_i)^2$,
\begin{align*}
  G_i = \frac{\d}{\d z} \nabla_\th \ell_i(\hat \th, z) \mid_{z=Z_i} = - \sigma \nabla_\th f(x_i; \hat \th)~.
\end{align*}
\subsection{Sensitivity of predictor to noise in the response at a local minimum}
The following key lemma is used in the proof of \Cref{thm:ci-hp}.
The lemma can be seen as an Implicit Function Theorem, or simply differentiating in the noise variable and applying the chain rule.%
\begin{lemma}
  \label{lem:norm2norm}
  Consider the collection of fixed loss functions $\ell_i : \R^p \times \cZ \mapsto \R$ which are continuously differentiable in the second argument.
  Then, assuming that $\hat \th\repi(z)$ is a local minimizer of $L_{\lambda}$, for any fixed $\xtest$, we have
  \begin{align*}
    \|\nabla_z f(\xtest; \hat \th\repi(z)\|^2 = \frac{1}{n^2} \|\nabla_\th f(\xtest; \hat \th\repi(z))\|_{M_i(z)}^2
  \end{align*}
  where
  \begin{align*}
    M_i(z) = \pr{\nabla^2 L(\hat \th\repi(z)) + \lambda I}^{-1} G_i(z) G_i(z)\tp \pr{\nabla^2 L(\hat \th\repi(z)) + \lambda I}^{-1}
  \end{align*}
  and
  \begin{align*}
    G_i(z) = \frac{\d}{\d z'} \nabla_\th \ell_i(\hat \th\repi(z), z') ~\bigg|_{z' = z}~.
  \end{align*}
\end{lemma}
\begin{proof}
Abbreviate $\hat \th(z) = \hat \th\repi(z)$.
Observe that by the chain rule
\begin{align*}
  \nabla_z f(\hat \th(z))
  =
  D_z [\hat \th(z)]\tp \nabla_\th f(\hat \th(z))
\end{align*}
where $D_z [\hat \th(z)] \in \R^{p \times d}$ is a Jacobian matrix and $\nabla_\th f(\hat \th(z)) \in \R^p$.

The rest of the proof deals with giving a closed-form expression for $D_z [\hat \th(z)]$.
Since $\hat \th(z)$ is a stationary point by \Cref{lem:alignment} we have
\begin{align*}
  &\lambda \, \hat \th(z) = - \nabla L^{\backslash z}(\hat \th(z)) - \frac1n \nabla_\th \ell_i(\hat \th(z), z)\\
  \implies\qquad
  &\lambda D_z [\hat \th(z)]
  =
  - \nabla^2 L^{\backslash z}(\hat \th(z)) \, D_z [\hat \th(z)]
  - D_z \br{\frac1n \nabla_\th \ell_i(\hat \th(z), z)}
\end{align*}
where we used notation $L^{\backslash z}(\th) = (1/n) \sum_{j \neq i} \ell_j(\th, Z_j)$.
In particular,
\begin{align*}
  D_z[\nabla_\th \ell_i(\hat \th(z), z)]
&=
    \nabla_\th^2 \ell_i(\hat \th(z), z) D_z [\hat \th(z)]
    +
    \frac{\d}{\d z'} \nabla_\th \ell_i(\hat \th(z), z') \bigg|_{z'=z}
\end{align*}
and so
\begin{align*}
  \lambda D_z [\hat \th(z)]
  =
  - \pr{\nabla^2 L^{\backslash z}(\hat \th(z)) + \frac1n \nabla_\th^2 \ell_i(\hat \th(z), z)} \, D_z [\hat \th(z)] - \frac1n \pr{\frac{\d}{\d z'} \nabla_\th \ell_i(\hat \th(z), z') \bigg|_{z'=z}}
\end{align*}
while rearranging and solving for $D_z [\hat \th(z)]$ (note that the regularized Hessian is invertible since $\hat \th(z)$ is a local minimizer), we have
\begin{align*}
  D_z [\hat \th(z)]
  =
  - \frac1n \pr{\nabla^2 L(\hat \th(z)) + \lambda I}^{-1} \pr{\frac{\d}{\d z'} \nabla_\th \ell_i(\hat \th(z), z') \bigg|_{z'=z}}~.
\end{align*}
\end{proof}
\subsection{Proof of \Cref{thm:ci-hp}}

The proof is based on the following bound on the exponential moment due to Pisier,
which tells that a differentiable function of a Gaussian
vector concentrates around its mean well if its gradient is
well-behaved:
\begin{lemma}[{\citet{pisier1986probabilistic}, \citet*[Display below Lemma 2.27]{wainwright2019high}}]
  \label{lem:F-concentration}
  Let $S=(Z_1, \ldots, Z_n) \sim \cN(0, I_n)$.
  Then, for a differentiable function $F : \R^n \to \R$ and any $\lambda \in \R$ we have
  \begin{align*}
    \E \exp\pr{\lambda |F(S) - \E[F(S)]|}
    \leq
    2 \E \exp\pr{\frac{\lambda^2 \pi^2}{8} \, \|\nabla F(S)\|^2}~.
  \end{align*}
\end{lemma}
\begin{lemma}[{\citet*[Lemma 11]{boucheron2003concentration}}]
  \label{lem:ci-hp-help-b}  
  Introduce $h(x)=(1-\sqrt{1 + 2 x})^2/2$.
  Let $A$ and $\beta$ denote two positive real numbers. Then,
$$
\sup_{\lambda \in [0, 1/a)} \left( \lambda t - \frac{A \lambda^2}{1 - a \lambda} \right) = \frac{2A}{a^2} h\left( \frac{a t}{2A} \right) \ge \frac{t^2}{2(2A + a t/3)}~.
$$
\end{lemma}
Combining the above two results, we have the following concentration inequality:
\begin{lemma}
  \label{lem:F-concentration-subgamma}
  Let $S=(Z_1, \ldots, Z_n) \sim \cN(0, I_n)$.
  Assume that $F : \R^n \to \R$ is a differentiable function.
  Then, assuming that $\|\nabla F(S)\|^2 - \E[\|\nabla F(S)\|^2]$ is a sub-gamma random variable with variance factor $v^2$ and scale $c$, for any $\delta \in (0,1)$ we have,
  \begin{align*}
    \P\cbr{|F(S) - \E[F(S)]| >
    \sqrt{\frac{\pi^2 \ln(2/\delta)}{2} \E[\|\nabla F(S)\|^2] + 2 \ln(2/\delta) v^2 + \frac{\ln^2\pr{2/\delta}}{9} \, c^2}
    +
    \frac{\ln(2/\delta)}{3} \, c
    } \leq \delta
  \end{align*}
\end{lemma}
\begin{proof}
  Consider 
  \Cref{lem:F-concentration}.
  Assuming that $\|\nabla F(S)\|^2$ is a sub-gamma random variable,
\begin{align*}
  \psi_{F(S) - \E[F(S)]}(\lambda)
  &:= \ln \E \exp\pr{\lambda \abs{F(S) - \E[F(S)]}}\\
  &\leq
    \frac{\lambda^2 \pi^2}{8} \E[\|\nabla F(S)\|^2] + \frac{\lambda^2 (v^2/2)}{1-c \lambda}\\
&\leq
    \frac{\lambda^2}{1-c \lambda} \pr{\frac{\pi^2}{8} \E[\|\nabla F(S)\|^2 + \frac{v^2}{2}}~.
\end{align*}
Using the Chernoff's method for any $t, \lambda > 0$
\begin{align*}
  \P\cbr{|F(S) - \E[F(S)]| > t}
  &\leq    
    2 \E \exp\pr{- \pr{\lambda t  - \psi_{F(S) - \E[F(S)]}(\lambda)}  }
\end{align*}
while focusing on the term in $\exp()$ on the right-hand side.
\begin{align*}
  \lambda t  - \psi_{F(S) - \E[F(S)]}(\lambda)
  &\geq
    \lambda t -
    \frac{\lambda^2}{1-c \lambda} \pr{\frac{\pi^2}{8} \E[\|\nabla F(S)\|^2 + \frac{v^2}{2}}
\end{align*}
and we see that $\lambda \in (0, 1] \cap (0, 1/c) = (0, 1/c)$.
Now we apply \Cref{lem:ci-hp-help-b} to maximize the right-hand side in the above over $\lambda \in (0, 1/c)$ and get
\begin{align*}
  \P\cbr{|F(S) - \E[F(S)]| > t}
  \leq
  2 \exp\pr{- \frac{t^2}{(\pi^2/2) \E[\|\nabla F(S)\|^2] + 2 v^2 + (2/3) c t}}
\end{align*}
while inverting this bound yields the statement.
\end{proof}

\begin{proof}[Completing proof of \Cref{thm:ci-hp}.]
  We will apply \Cref{lem:F-concentration-subgamma} with $F(S) = f(\hat \th)$, but first observe that
  \begin{align*}
    \nabla F(S)
    = \pr{\frac{\d}{\d z} f(\hat \th^{(1)}(z)) \bigg|_{z = Z_1}, \ldots, \frac{\d}{\d z} f(\hat \th^{(n)}(z)) \bigg|_{z = Z_n}}
  \end{align*}
and so by \Cref{lem:norm2norm}
\begin{align*}
    \|\nabla F(S)\|^2
    = \sumin \abs{\frac{\d}{\d z} f(\hat \th\repi(z)) \bigg|_{z = Z_i}}^2
    = \frac{1}{n^2} \sumin \|\nabla f(\hat \th\repi(Z_i))\|_{\hat M_i(Z_i)}^2
  \end{align*}
  while taking expectation on both sides and using the fact that $Z_i$'s are identically distributed
  \begin{align*}
    \E \|\nabla F(S)\|^2
    =
    \frac{1}{n^2} \sumin \E \|\nabla f(\hat \th\repi(Z_i))\|_{M_i(Z_i)}^2
    =
    \frac1n \E \|\nabla f(\hat \th)\|_{M(\hat \th)}^2~.
  \end{align*}
By assumption $\|\nabla f(\hat \th)\|_{M(\hat \th)}^2 - \E[\|\nabla f(\hat \th)\|_{M(\hat \th)}^2]$ is a sub-gamma random variable with variance factor $v^2$ and scale $c$, and so $\|\nabla F(S)\|^2 - \E[\|\nabla F(S)\|^2]$ is a sub-gamma random variable with $(v^2/n^2, c/n)$.
  Then, \Cref{lem:F-concentration-subgamma} gives us
    \begin{align*}
    \P\cbr{|f(\hat \th) - \E[f(\hat \th)]| >
    \sqrt{\frac{\pi^2 \ln(2/\delta)}{2 n} \E[\|\nabla f(\hat \th)\|^2_{M(\hat \th)}] + \frac{2 \ln(2/\delta) v^2}{n^2} + \frac{\ln^2\pr{2 / \delta}}{9 n^2} \, c^2}
    +
    \frac{\ln(2/\delta)}{3 n} \, c
    } \leq \delta~.
  \end{align*}
  The proof is now complete.
\end{proof}

 \section{Empirical validation}
\label{sec:experiments}

\newcommand{\methodname}{vhat}

This section details the experimental setup used to evaluate the proposed pointwise confidence estimation method (\Cref{thm:ci-hp-ls} where weighted norm is computed by \Cref{alg:weighted-norm}) against a standard baseline in a controlled regression setting with distributional shift.
We aim to verify the hypothesis that \Cref{thm:ci-hp-ls} results in confidence intervals with good coverage of a non-linear ground-truth function, which are still narrow in the support of the training data.
In particular, we are interested in its behaviour in regions with a lack of the training data.
We will start by looking at a one-dimensional regression example, where we train a neural network to fit the data.
Despite its simplicity, this setting is meaningful because the sensitivity of the confidence bound to the test point
(as captured by the weighted norm)
is measured in a high-dimensional implicit feature space (in other words, the neural network first projects a scalar input into a high-dimensional feature space).
Then, we will look at the performance with inputs  of a higher dimension.

\paragraph{Data}

We consider regression problems of various input dimensions.
Inputs are drawn from a standard normal distribution, $\mathcal{N}(0, I_d)$.
The corresponding output $Y$ is generated according to the model \cref{eq:regression}, where the ground-truth function $f(x; \th^{\star})$ is a function sampled (and fixed throughout) from a Mat\'ern \ac{RKHS}~\citep{RasmussenW06}, providing a smooth, non-linear target function.
For all experiments, the length scale, output scale, and smoothness ($\nu$) parameters are respectively $(1,1,1.5)$.\footnote{For experiments with $d > 1$ with set output scale to $0.1$.}
The standard deviation of the regression noise is $\sigma = 0.1$.%

A key aspect of our experimental design introduces a distributional shift between training and evaluation data, which is inspired by \citep{antoran2022adapting}.
Once the training and validation sets are drawn as described above, we remove all examples whose inputs lie within a box $[-0.5,0.5]^d$. %
The validation set used for hyperparameter search is sampled in the same way as the training set.
Meanwhile, the test inputs are either drawn from the original Gaussian distribution, without removing points (for $d > 1$) or form a uniform grid (for $d=1$).
This way we are able to test sensitivity of the confidence bound to the region of the input space where no observations were made.%

\paragraph{Model and training}
We employ a fully connected neural network with 2 hidden layers ($1024$ and $512$ neurons respectively) and the \texttt{ReLU} activation function in the case of the $d=1$ experiments, and a neural network with $1024$-sized hidden layer for $d > 1$.
The output layer consists of a single neuron.

For $d=1$ we use variance scaling with standard deviation $100$ (this results in a less smooth predictor), while for $d > 1$
we initialize parameters using the Glorot scheme and train a neural network by minimizing the average square loss with $\ell_2$ regularization ($L_{\lambda}$).
The optimization is performed using the AdamW optimizer \citep{loshchilov2018decoupled} with parameters with a step size $\eta=10^{-3}$, and weight decay $\lambda = 10^{-8}$ for $d=1$, $\lambda = 10^{-3}$ for $d = 10$, and $\lambda = 10^{-4}$ for $d=100$.
We utilized a cosine learning rate schedule over the course of training.
The model is trained for a total of $10^4$ passes over the training data (without mini-batching).
All experiments are performed on a platform with an Nvidia H$100$ GPU.

The hyperparameter $\lambda$ is tuned by minimizing the mean Winkler interval score (see below) on the validation set.
Experiments in \Cref{tab:d10,tab:d100} are performed on $5$ independent trials with respect to draws of the data, while the tables report metric averages or medians with standard deviations over these trials.

\paragraph{Confidence estimation}
We evaluate pointwise confidence intervals generated by our proposed method. As a baseline for comparison, we consider the standard non-parametric bootstrap method \citep{efron1994introduction}. For the bootstrap, we generated $10$ bootstrap replicates of the training dataset sampled with replacement, trained a separate network on each replicate using the same architecture and training procedure, and derived empirical confidence intervals, using the quantile method with $\alpha = 0.01$, from the resulting ensemble of predictions at each test point.
Confidence intervals constructed from the bootstrap are a good baseline as many other methods based on ensembles can be viewed as a proxy to it~\citep{lakshminarayanan2017simple,gal2016dropout,osband2021epistemic}.
The `proposed method' is a  bound in \Cref{thm:ci-hp-ls} with $\nu = c = 1$ and $\delta = 0.01 / \text{\# test points}$ (where the latter divison accounts for a union bound over test points).
In particular, to compute the weighted norm $\|\nabla_{\th} f(\xtest; \hat \th)\|_{M(\hat \th)}^2$ we used \Cref{alg:weighted-norm} with $10^3$ \ac{CG} iterations and tolerance $10^{-12}$ (we noticed that \ac{CG} often requires a $64$-bit precision to converge reliably).
\paragraph{Evaluation}
The primary goal is to assess the quality of the pointwise confidence intervals, particularly within the unobserved regions.
For the $d=1$ case, we visualize the results in \Cref{fig:vhat}.
In this case, we provide plots for various sample sizes.
The expectation is that a reliable estimate should yield wider CIs in regions without observations compared to regions well-covered by training data.

In problems with $d > 1$, we rely on metrics such as coverage, width, and their combination as a single metric known as the mean Winkler interval score~\citep{winkler1972decision}.
Here, coverage is computed as an average of indicators
$\ind\{\textsc{lb}({\xtest}) \leq f(\xtest; \th^{\star}) \leq \textsc{ub}({\xtest})\}$
over all test inputs $\xtest$ where $\textsc{ub}()$ and $\textsc{lb}()$ are upper and lower bounds.
The width is computed as
an average of widths $\textsc{ub}(x_{\text{nn}}(\xtest)) - \textsc{lb}(x_{\text{nn}}(\xtest))$ over all test inputs $\xtest$, where $x_{\text{nn}}(\xtest)$ is a nearest neighbour of $\xtest$ among all training inputs (but with filtering out outliers).\footnote{Let $\ve$ be a $99$th percentile of
$\cbr{\min_j \|x_i - x_j^{\text{test}}\| ~:~ i \in [n]}$.
Then, evaluation is done on all inputs indexed by
$\cbr{\argmin_j \|x_i - x_j^{\text{test}}\| ~:~ i \in [n]~, \, \min_j \|x_i - \xtest_j^{\text{test}}\| \leq \ve}$.}
This ensures that we compute the width only in the regions where we have sufficiently many observations.
Finally, the mean Winkler interval score is defined as %
$ W_{\alpha} = (1/n) \sum_i [ (\textsc{ub}(x_i) - \textsc{lb}(x_i)) + (2/\alpha) \max(0, Y_i - \textsc{ub}(x_i)) + (2/\alpha) \max(0, \textsc{lb}(x_i) - Y_i)  ]
$ where $\alpha = 1 - \text{coverage}$. %
Clearly, $W_{\alpha}$ is small when we have a good coverage, yet intervals are narrow,
and it can be computed on any sample, which makes it useful for hyperparameter selection.
\paragraph{Discussion}
First, focusing on \Cref{fig:vhat}, we indeed observe that the proposed method yields wide intervals in regions where few to none observations were made (see left and right sides of the graph as well as the central `cut-out' part), thus confirming our hypothesis on weighted norms.
At the same time, intervals provide an adequate coverage where many observations are available, despite the fact that we do not control the bias term.
Second, confidence intervals seem to narrow with increase of the sample size, which is an expected behavior.

In contrast, and as we can see in the second row of \Cref{fig:vhat}, bootstrap fails to produce a good confidence interval on this task even when $n$ is large.
This is a common problem with bootstrap, which might yield an overconfident intervals and would require a correction term which is not provided by asymptotic confidence intervals.
Looking at higher dimensional cases in \Cref{tab:d10,tab:d100} and \Cref{fig:pareto}, we note that the median width is decreasing with an increase in the sample size, while maintaining a good coverage of the ground truth.
In contrast, similarly to the $d=1$ case, bootstrap confidence intervals again demonstrate over-confidence, which results in a reduced coverage.
Finally, \Cref{fig:pareto} is intended to gauge the sensitivity of the proposed method to the choice of the regularization parameter $\lambda \in \{10^{-6}, 10^{-5}, \ldots, 10^{-2}\}$.
This figure also captures the trade-off between coverage and CI width, with the best outcome being the top-left corner.
It is evident that $\lambda$ controls this trade-off and can be tuned in practice to obtain the trade-off as close as possible to the top-left corner.
On the other hand, bootstrap appears to be insensitive to various choices of $\lambda$ possibly because the underlying neural network bootstrap replications are not much affected by $\lambda$ in terms of their smoothness.

\newcommand{\figuresize}{0.3}
\begin{figure} \center
\begin{subfigure}{\figuresize\textwidth}
\includegraphics[width=\textwidth]{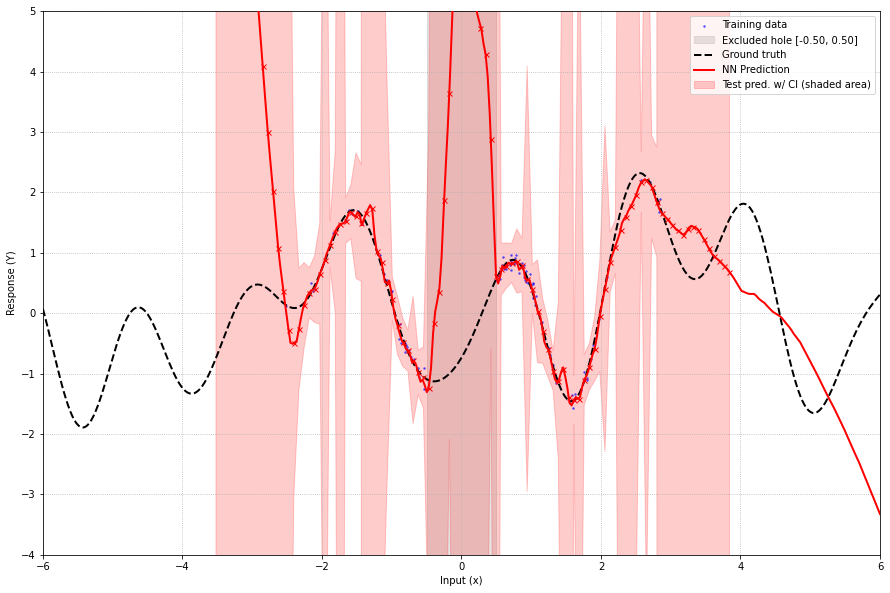}
\end{subfigure}
\begin{subfigure}{\figuresize\textwidth}
\includegraphics[width=\textwidth]{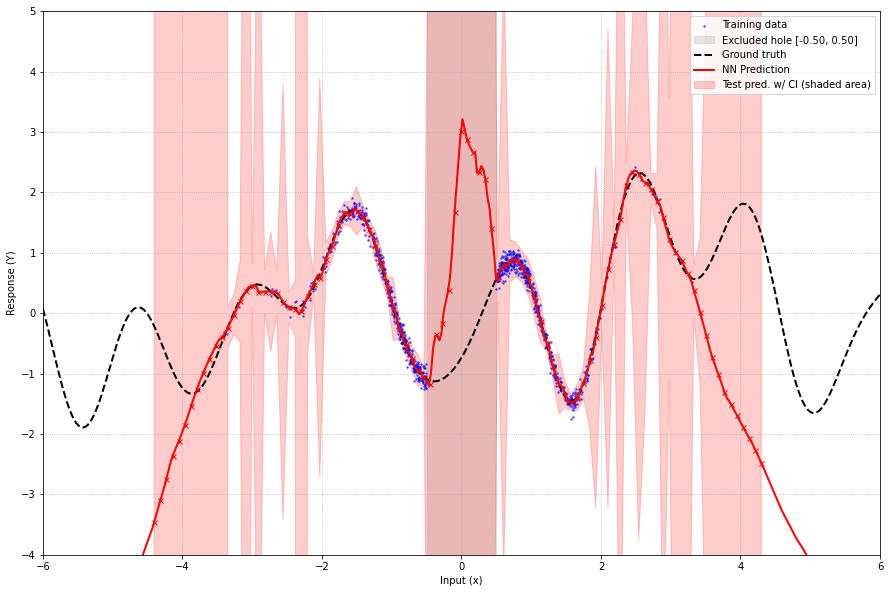}
\end{subfigure}
\begin{subfigure}{\figuresize\textwidth}
\includegraphics[width=\textwidth]{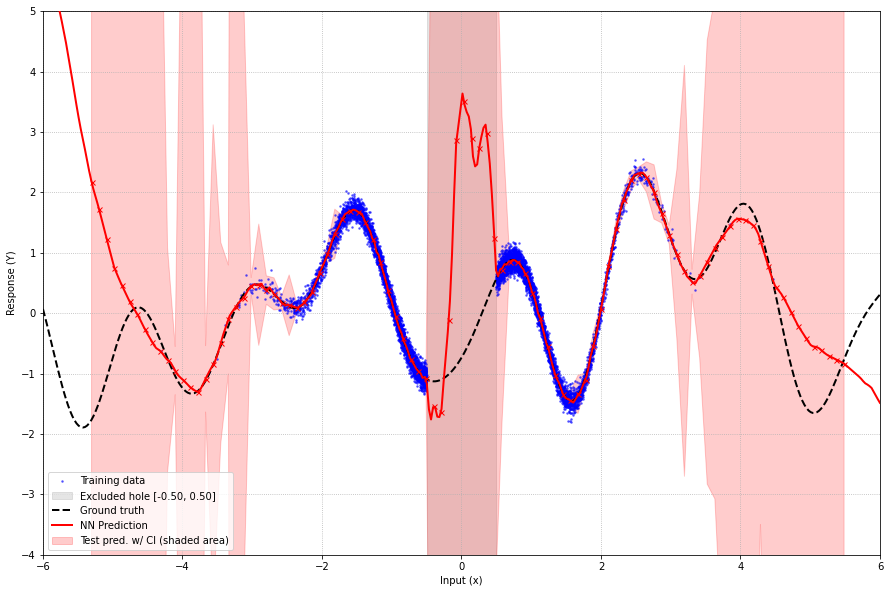}
\end{subfigure}
\begin{subfigure}{\figuresize\textwidth}
\includegraphics[width=\textwidth]{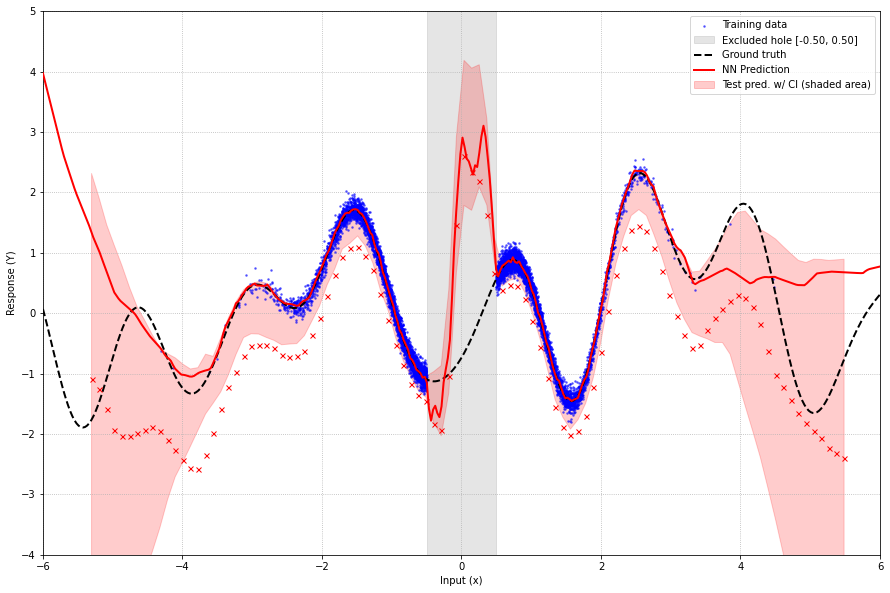}
\caption*{$n = 100$}
\end{subfigure}
\begin{subfigure}{\figuresize\textwidth}
\includegraphics[width=\textwidth]{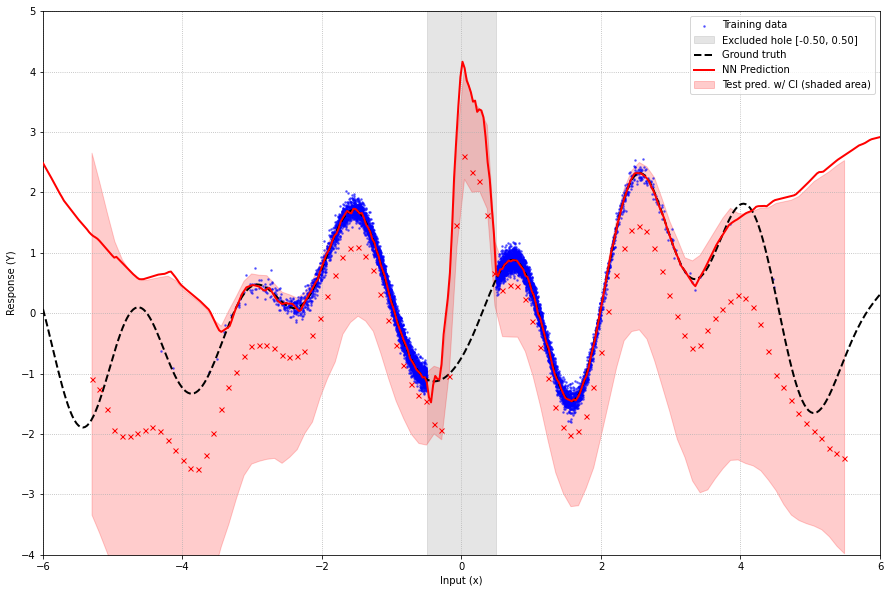}
\caption*{$n = 1000$}
\end{subfigure}
\begin{subfigure}{\figuresize\textwidth}
\includegraphics[width=\textwidth]{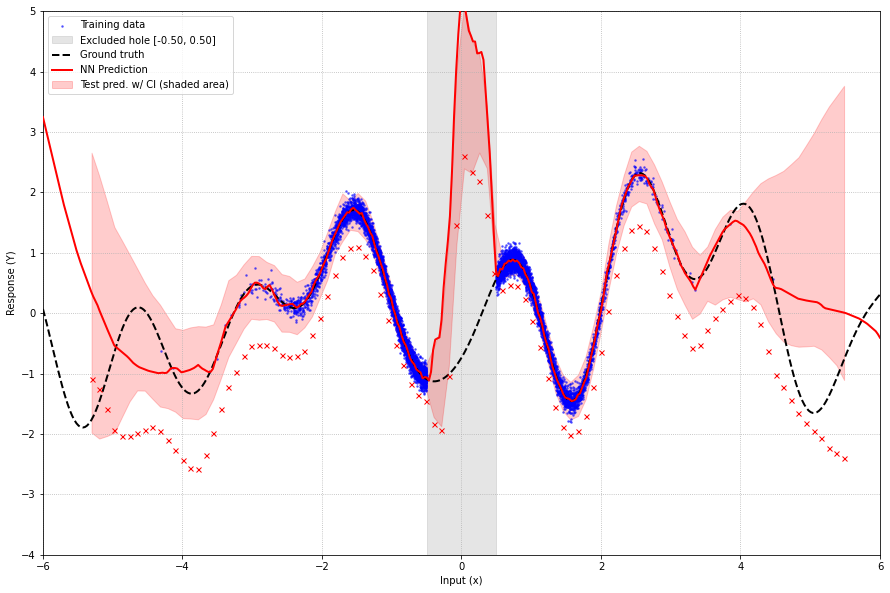}
\caption*{$n = 10000$}
\end{subfigure}
\caption{The first row: proposed confidence interval evaluation with increasing sample size.
  The second row: bootstrap confidence interval evaluation. In this case confidence intervals are formed from quantiles over predictions from $10$ models trained on draws from the training sample with replacement.
}
\label{fig:vhat}
\end{figure}

\renewcommand{\arraystretch}{1.2} \newcommand{\scaletables}{0.75}

\begin{table}[H] \centering
  \caption{Proposed method vs bootstrap quantile CI on a $10$-dimensional problem, $\delta = 0.01$, with $\lambda=10^{-3}$ in all cases.}
  \label{tab:d10}
\scalebox{\scaletables}{
\begin{tabular}{l | l cccccc}
\toprule
& N Train & Test MSE & Avg CI & Median CI & Mean Interval Score & CI Cover \\
&         &          & Width  & Width     & ($W_{0.01}$)      & (\%)     \\
\midrule
\multirow{4}{*}{\rotatebox[origin=c]{90}{Proposed}} &
100   & $0.0312 \pm 0.0046$ & $2.1295 \pm 1.2149$ & $0.5162 \pm 0.0920$ & $2.3052 \pm 1.0058$ & $99.60\% \pm 0.49\%$ \\
& 1000  & $0.0246 \pm 0.0054$ & $0.4846 \pm 0.0120$ & $0.2418 \pm 0.0131$ & $2.9416 \pm 1.1417$ & $89.80\% \pm 2.32\%$ \\
& 10000 & $0.0270 \pm 0.0043$ & $1.0030 \pm 0.5659$ & $0.3010 \pm 0.0840$ & $2.1816 \pm 1.1668$ & $96.00\% \pm 4.77\%$ \\
& 20000 & $0.0233 \pm 0.0049$ & $0.6623 \pm 0.3477$ & $0.2333 \pm 0.0891$ & $3.5615 \pm 2.3433$ & $89.60\% \pm 10.67\%$ \\
\midrule
\multirow{4}{*}{\rotatebox[origin=c]{90}{Bootstrap}} &
100   & $0.0329 \pm 0.0059$ & $0.2356 \pm 0.0246$ & $0.1132 \pm 0.0125$ & $9.7120 \pm 0.7665$ & $62.25\% \pm 4.44\%$ \\
& 1000  & $0.0246 \pm 0.0022$ & $0.1724 \pm 0.0063$ & $0.0820 \pm 0.0036$ & $10.4633 \pm 0.9745$ & $56.60\% \pm 4.45\%$ \\
& 10000 & $0.0332 \pm 0.0027$ & $0.3059 \pm 0.0077$ & $0.1453 \pm 0.0041$ & $5.6643 \pm 0.9700$ & $82.80\% \pm 4.31\%$ \\
& 20000 & $0.0270 \pm 0.0033$ & $0.2964 \pm 0.0072$ & $0.1401 \pm 0.0035$ & $5.4321 \pm 0.6177$ & $83.00\% \pm 2.10\%$ \\
\bottomrule
\end{tabular}
}
\end{table}

\begin{table}[H] \centering
\caption{Same setting as in \Cref{tab:d10} but on a $100$-dimensional problem, with $\lambda=10^{-4}$ in all cases.}
  \label{tab:d100}
  \centering  
  \scalebox{\scaletables}{
\begin{tabular}{l | l cccccc}
\toprule
& N Train & Test MSE & Avg CI & Median CI & Mean Interval Score & CI Cover \\
&         &          & Width  & Width     & ($W_{0.01}$)      & (\%)     \\
\midrule
\multirow{4}{*}{\rotatebox[origin=c]{90}{Proposed}} &
100   & $0.0734 \pm 0.0071$ & $3.4774 \pm 4.5786$ & $0.7123 \pm 0.7343$ & $8.2162 \pm 2.3118$ & $84.40\% \pm 9.52\%$ \\
& 1000  & $0.0406 \pm 0.0053$ & $0.6456 \pm 0.1690$ & $0.2281 \pm 0.0658$ & $7.6975 \pm 1.2885$ & $79.60\% \pm 2.73\%$ \\
& 10000 & $0.0279 \pm 0.0027$ & $0.7453 \pm 0.3018$ & $0.3425 \pm 0.0777$ & $1.9889 \pm 0.2441$ & $96.00\% \pm 2.19\%$ \\
& 20000 & $0.0268 \pm 0.0035$ & $0.6712 \pm 0.1748$ & $0.3042 \pm 0.0083$ & $1.7652 \pm 0.3218$ & $94.40\% \pm 1.36\%$ \\
\midrule
\multirow{4}{*}{\rotatebox[origin=c]{90}{Bootstrap}} &
100   & $0.0666 \pm 0.0053$ & $0.2128 \pm 0.0141$ & $0.1029 \pm 0.0077$ & $23.5971 \pm 3.0576$ & $31.20\% \pm 4.79\%$ \\
& 1000  & $0.0406 \pm 0.0067$ & $0.1773 \pm 0.0072$ & $0.0856 \pm 0.0044$ & $18.0165 \pm 1.7836$ & $42.20\% \pm 7.39\%$ \\
& 10000 & $0.0345 \pm 0.0035$ & $0.2348 \pm 0.0048$ & $0.1141 \pm 0.0013$ & $11.1751 \pm 1.4537$ & $59.20\% \pm 4.45\%$ \\
& 20000 & $0.0309 \pm 0.0060$ & $0.2254 \pm 0.0030$ & $0.1091 \pm 0.0024$ & $10.4342 \pm 2.1573$ & $60.80\% \pm 5.78\%$ \\
\bottomrule
\end{tabular}    
}
\end{table}

\newcommand{\paretofiguresize}{0.25}
\begin{figure}[H]
  \center
\begin{subfigure}{\paretofiguresize\textwidth}
\includegraphics[width=\textwidth]{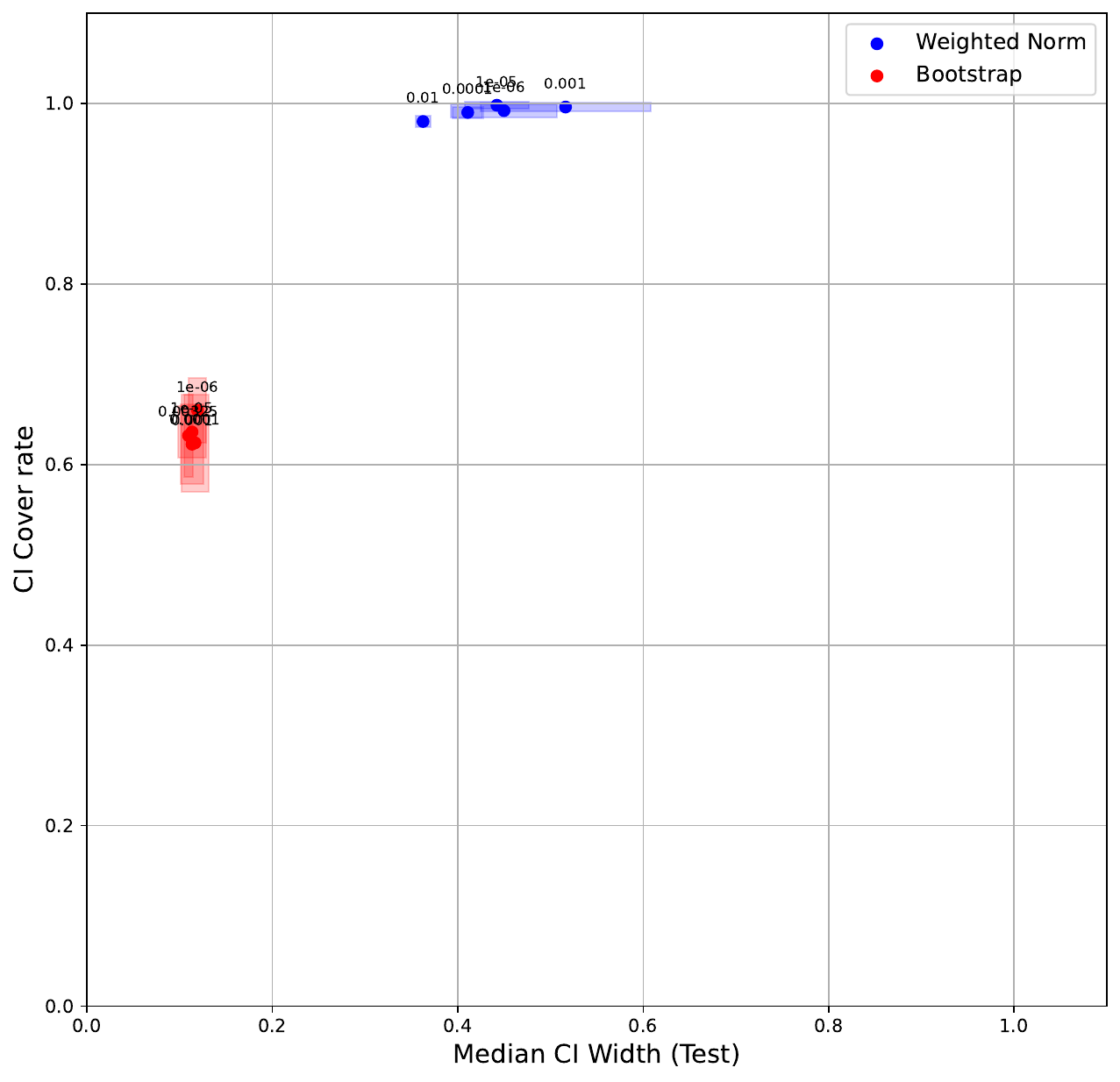}
\end{subfigure}
\begin{subfigure}{\paretofiguresize\textwidth}
\includegraphics[width=\textwidth]{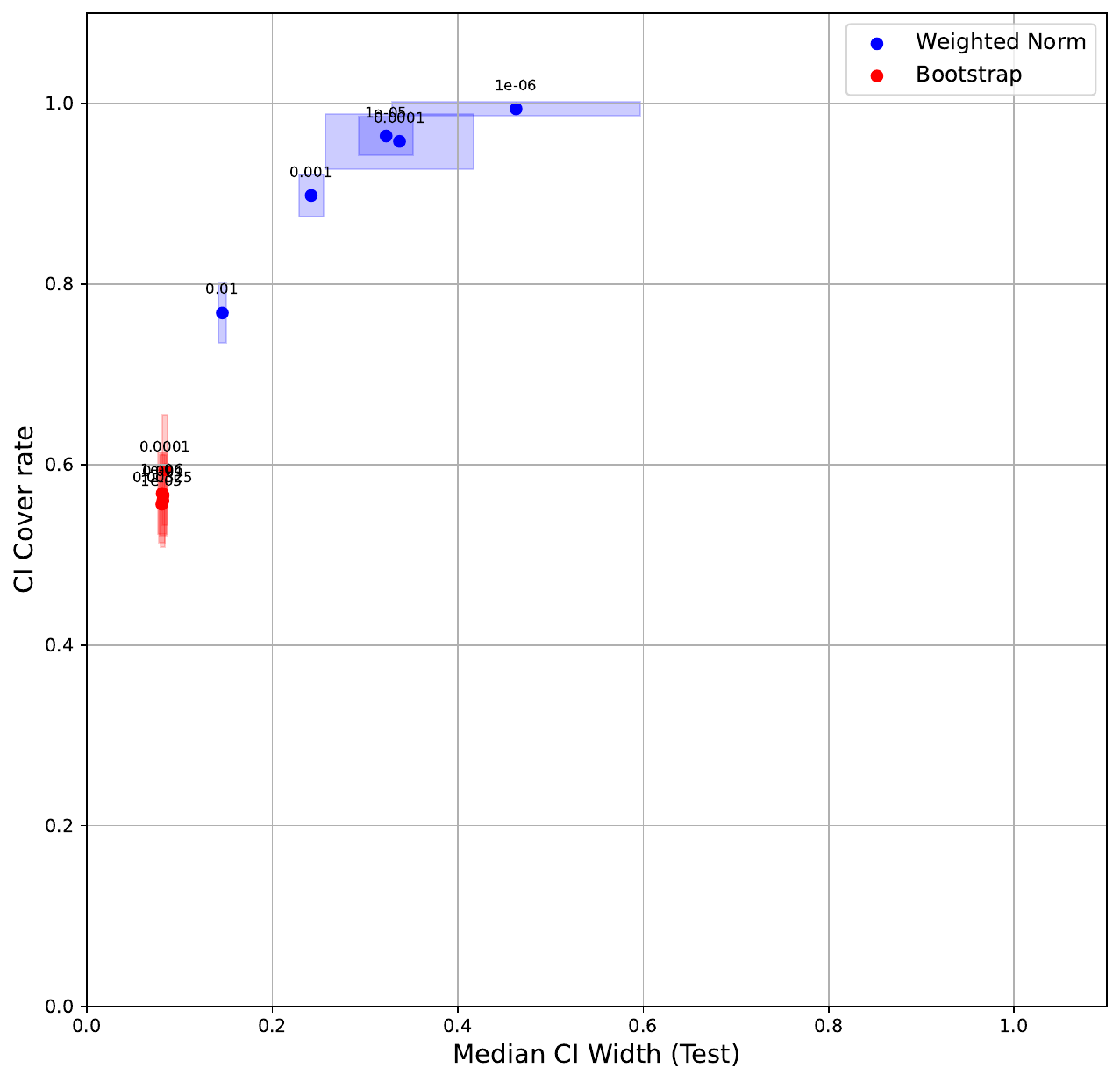}
\end{subfigure}
\begin{subfigure}{\paretofiguresize\textwidth}
\includegraphics[width=\textwidth]{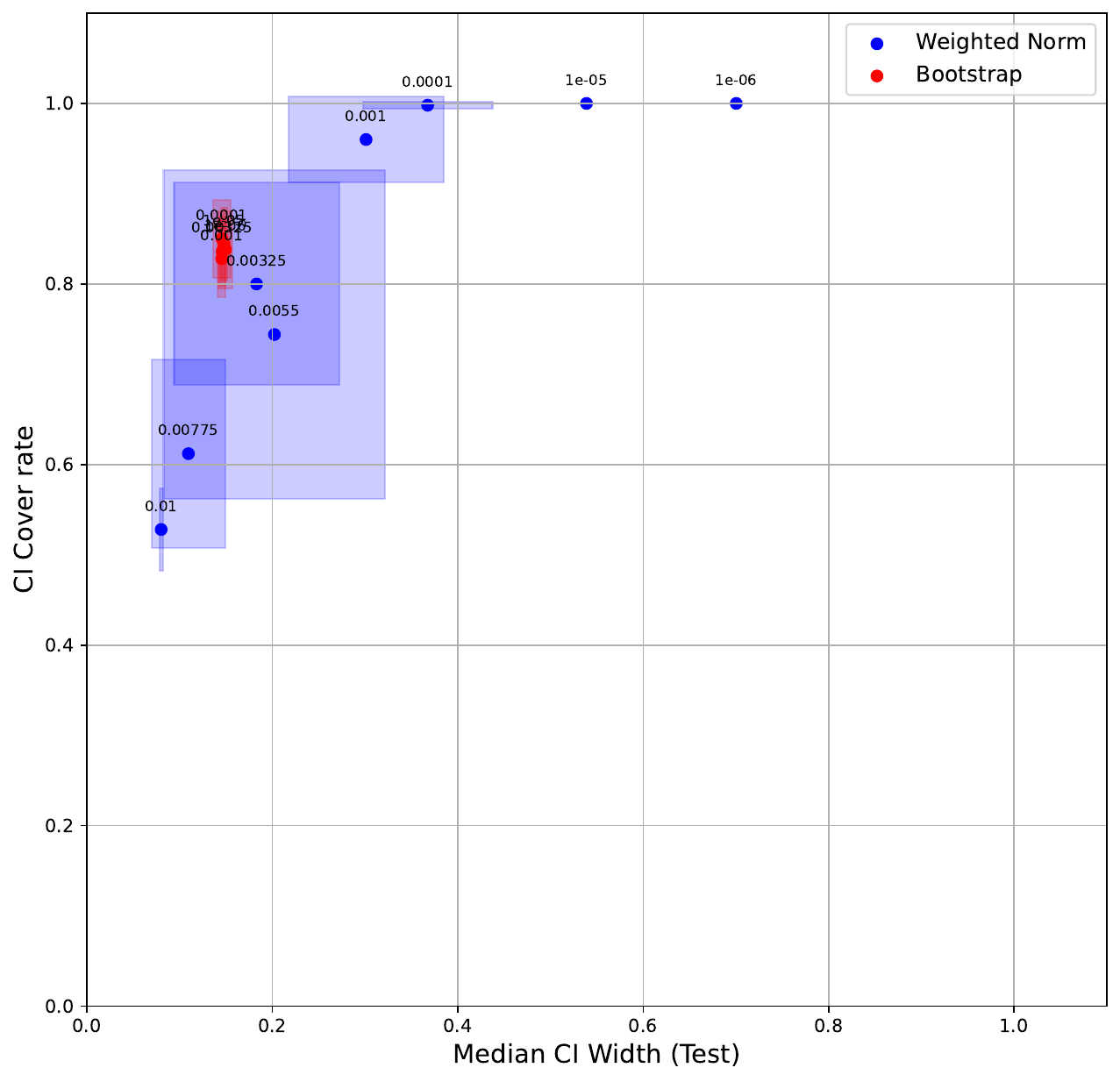}
\end{subfigure}
\begin{subfigure}{\paretofiguresize\textwidth}
\includegraphics[width=\textwidth]{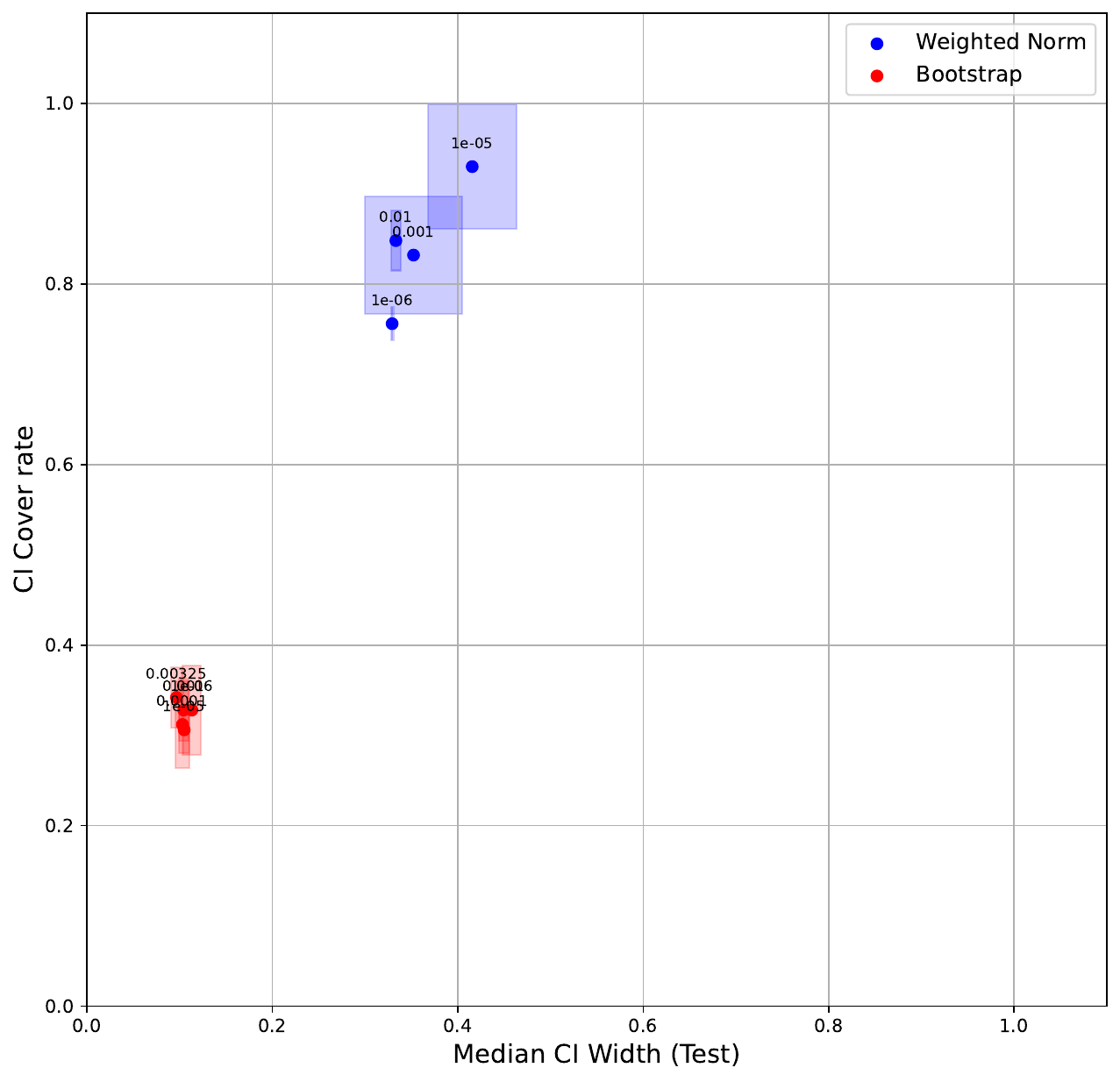}
\caption*{$n = 100$}
\end{subfigure}
\begin{subfigure}{\paretofiguresize\textwidth}
\includegraphics[width=\textwidth]{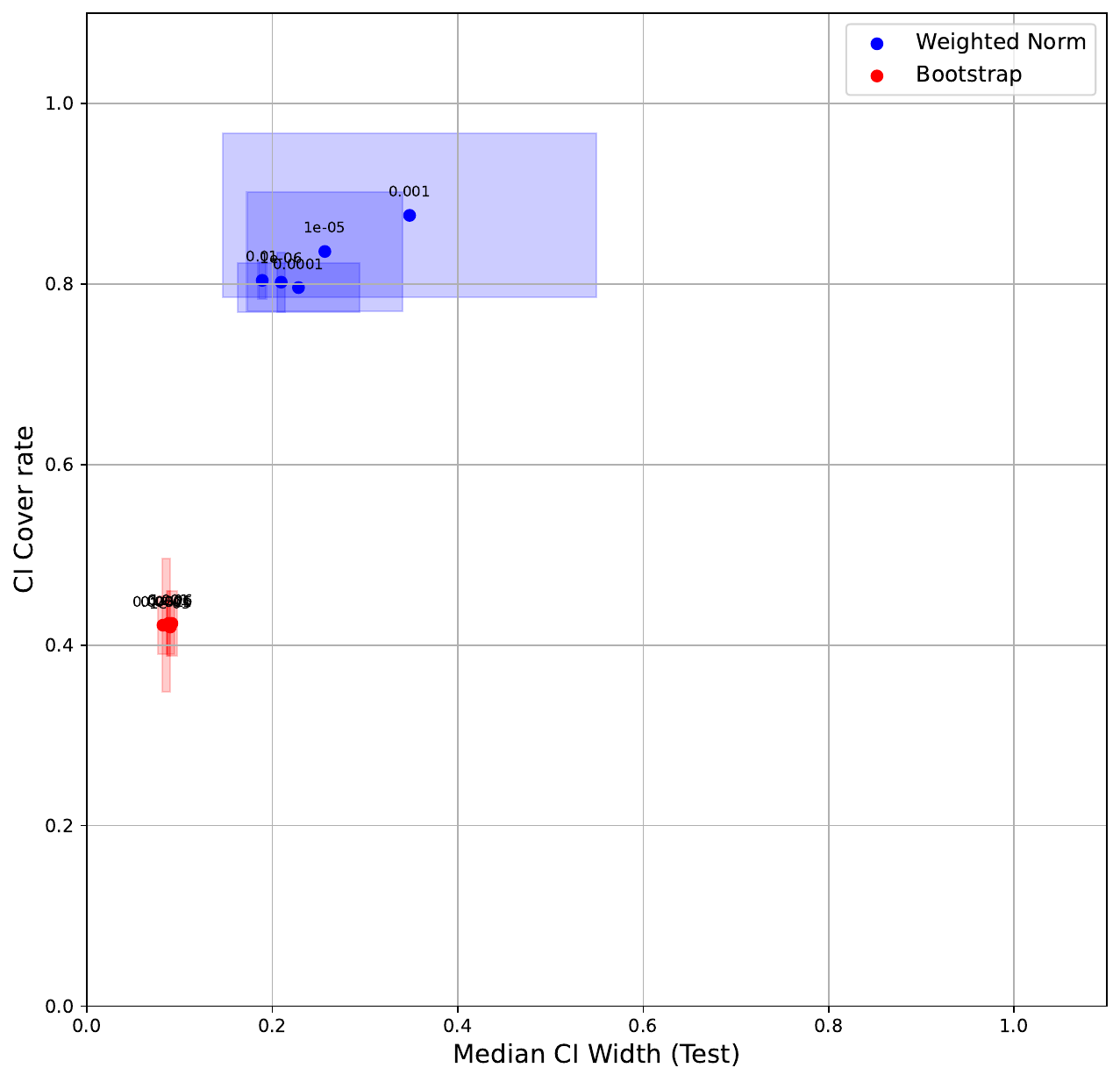}
\caption*{$n = 1000$}
\end{subfigure}
\begin{subfigure}{\paretofiguresize\textwidth}
\includegraphics[width=\textwidth]{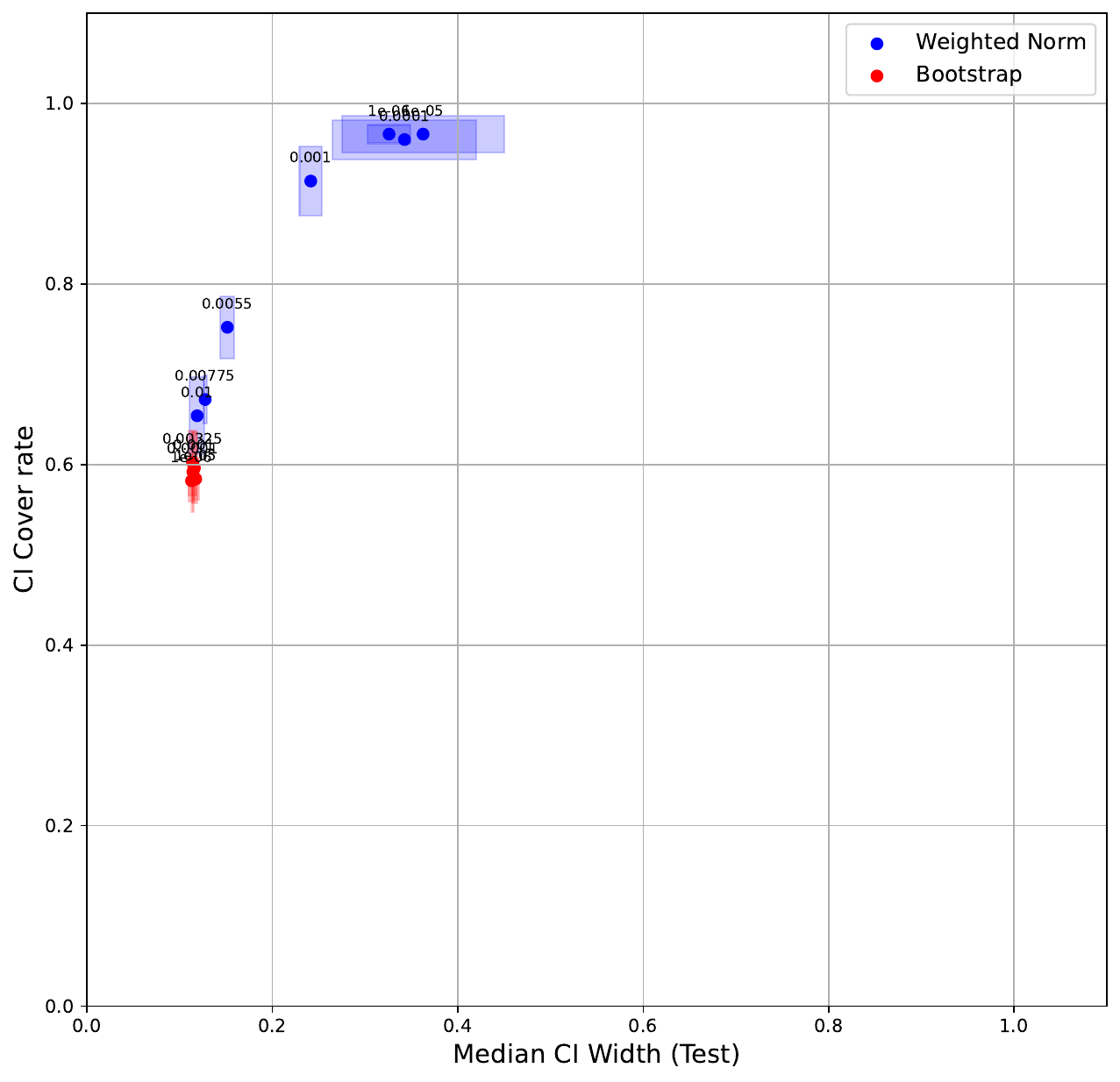}
\caption*{$n = 10000$}
\end{subfigure}
\caption{Trade-off between CI width and coverage,
  when varying regularization parameter $\lambda \in \{10^{-6}, 10^{-5}, \ldots, 10^{-2}\}$.
  The best outcome is the top-left corner.
  Each figure captures median CI width on the test set (horizontal axis) against coverage rate (vertical axis).
  Each point corresponds to evaluation given a specific regularization parameter $\lambda$ (the text label attached to each point is $\lambda$).
  In particular, a point is an average of evaluations, while the shaded the region represents standard deviations.
  Blue points stand for the proposed method while red points are bootstrap.
  The first row corresponds to $d=10$ while the second to $d=100$.
}
\label{fig:pareto}
\end{figure}

\clearpage
\bibliographystyle{plainnat}
\bibliography{learning}

\clearpage
\appendix
\section{Additional related work}
\label{sec:related}
\textbf{\emph{Confidence intervals from asymptotic normality of MLE}}
The standard theory of parametric statistics establishes that \ac{MLE}, $\hat{\theta}_n$, are asymptotically normal~\citep{van2000asymptotic}.
This pivotal result is often stated as $\sqrt{n}(\hat{\theta}_n - \theta^{\star}) \to \cN(0, \cI(\theta^{\star})^{-1})$ in distribution,
where $\theta^{\star}$ is the true parameter value and $\cI(\theta^{\star})$ is the Fisher Information Matrix. One of the most prevalent confidence sets based on this result are \emph{Wald-type confidence sets} that typically take an ellipsoidal form for a vector parameter $\theta$:
$$ \cbr{ \theta \in \Theta : n(\hat{\theta}_n - \theta)\tp \hat{\cI}_n (\hat{\theta}_n - \theta) \leq \chi^2_{k, \alpha} } $$
where $\chi^2_{k, \alpha}$ is the upper $\alpha$-quantile of the chi-squared distribution with $k$ degrees of freedom (the dimension of $\theta$) and where $\hat{\cI}_n$ is an estimate of the Fisher Information Matrix $\cI(\theta^{\star})$.

Despite their simplicity, Wald-type confidence sets often exhibit poor finite sample performance, leading to unreliable coverage probabilities, and can produce intervals/sets that extend beyond valid parameter boundaries, particularly when the true parameter is near a boundary. Furthermore, they lack invariance to parameter transformations, which can lead to different inferential conclusions based on the chosen parameterization.

In this work we do not explicitly construct confidence sets for parameters, but instead give confidence intervals for the predictor on a given test point directly.
However, our results are related to asymptotic normality of MLE, in a sense that the shape of the confidence interval is elliptical in nature as well, and is characterized by the weighting matrix that is different from $\cI(\th^{\star})$ (but matches it asymptotically under certain conditions).
One major distinction is that our result is non-asymptotic and so must enjoy better coverage in the presence of a small sample.
In \Cref{sec:connection-to-mle} we discuss two approaches more formally.

\textbf{\emph{Bootstrap}}
is a classical tool in statistics~\citep{efron1994introduction}, often treated as a gold-standard method for the construction of confidence sets when dealing with complex estimation problems where no other well-known alternatives are available.
The construction of confidence sets using bootstrap is often done through quantiles.
The fundamental principle involves approximating the unknown sampling distribution of an estimator, say $\hat{\theta}_n$,
with the empirical distribution of bootstrap replications. Confidence intervals are then formed by taking the appropriate quantiles from this bootstrap distribution. For instance, the percentile method, a direct quantile-based approach, uses the $\alpha/2$ and $1-\alpha/2$ quantiles of the bootstrap distribution of $\hat{\theta}_n^*$ to form a $1-\alpha$ confidence interval.
Standard textbooks such as \citep{van2000asymptotic,shao2003mathematical,wasserman2004all} establish the asymptotic validity of such methods, demonstrating that under regularity conditions, the bootstrap quantiles asymptotically mimic the quantiles of the true sampling distribution, leading to valid confidence sets. This method is particularly attractive because it is nonparametric and does not require explicit knowledge of the estimator's limiting distribution, making it broadly applicable even when analytic forms are intractable.
On the other hand, the bootstrap may fail to provide accurate coverage in small samples or in cases where the estimator is not asymptotically normal, as the convergence of the bootstrap distribution to the true distribution can be slow or even inconsistent in such settings. Moreover, the bootstrap can be sensitive to the presence of \emph{bias} and to the structure of the data, such as dependence or heavy tails, which may result in under- or overcoverage of the true parameter. In particular, the percentile method (using direct quantiles of the bootstrap distribution) can be suboptimal. There are alternative approaches, such as bias-corrected or studentized intervals, which may yield improved performance.
Another major limitation of bootstrap is the need to have a sufficiently many replications to reduce the variance,
which is extremely prohibitive in many applications, especially involving neural networks with millions or billions of parameters.

\textbf{\emph{Ensemble methods}} represent a family of techniques that draw upon the concept of bootstrap for estimating confidence. These methods involve training multiple predictors, each on a slightly altered version of the original data (for a comprehensive review on ensembling, see also \citep{dwaracherla2022ensembles}). Common approaches include utilizing different subsamples of the dataset for each model (bootstrap) and perturbing the observations \citep{osband2016deep}, or directly perturbing the model parameters~\citep{lakshminarayanan2017simple,gal2016dropout,osband2018randomized,osband2021epistemic}.
The empirical distribution derived from the resulting estimates is subsequently employed to construct confidence intervals.
While many of these methods can be viewed as sample-based approximations to Bayesian approaches, the selection of model hyperparameters (such as the scale of perturbations or the learning rate for models) in ensemble methods is typically performed using validation on withheld data (a subset of the training data).
Many of these methods are heuristic and seldom offer rigorous bounds for the construction of confidence sets.
However, while good performance might be expected when test inputs share a similar distribution with the training set, the resulting confidence intervals for test points distant from the training data may be unreliable.
Similar to bootstrap, these methods necessitate training multiple predictors from scratch, a requirement that does not align well with applications involving a large number of parameters.

\textbf{\emph{Bayesian methods}} are designed by assuming a prior distribution over the ground-truth (regression) function. This prior, when combined with observed data through a specified likelihood function, yields a posterior distribution. From this posterior, one can then compute \emph{credible intervals}~\citep{berger1985statistical}. Note that this is a rather different from what is considered in this work.
Prominent methods within the machine learning literature in this category include Bayes-by-Backprop \citep[BBB,][]{blundell2015weight}, the linearized Laplace method~\citep{Mackay1992Thesis,Immer2021Improving,antoran2022adapting}, and Stochastic Gradient MCMC \citep{nemeth2021stochastic}.

The linearized Laplace method is related to confidence intervals obtained from asymptotic normality of MLE, and to our approch, as it involves weighting by the inverse of the (regularized) Hessian matrix.
Note that (similarly to the case of the normality of the MLE) these results are asymptotic and rely on the Laplace approximation of the posterior.
Moreover, these results generally do not argue about high-probability credible intervals.
Here, in contrast, we develop high-probability non-asymptotic bounds from the frequentist viewpoint that do not require any approximations at all.
Interestingly, the shape of the weighting matrix is slightly different in our case.

\textbf{\emph{Confidence intervals for non-linear least-squares by solving constrained optimization problems}}
Several works explore an approach where the prediction of the ground truth function on $\xtest$ lies in the interval
$[\min_{\th \in \hat \Theta} f(\xtest; \th), \max_{\th \in \hat \Theta} f(\xtest; \th)]$ with high probability for appropriately chosen data-dependent confidence set $\hat \Theta$.
These methods were explored in the context of non-linear least squares~\citep{russo2014learning,foster2018practical,kong2021online,ayoub2020model}. However, in all of these cases, parameter confidence sets are constructed based on the conservative (data-free) covering number estimates. While this approach proved useful in the design of reinforcement learning and bandit algorithms for under-parameterized problems, for overparameterized problems this leads to vacuous estimates. A related line of work concerns the construction of confidence intervals for non-linear ground truth function, assuming that they lie in the \ac{RKHS} of a kernel function~\citep{chowdhury2017kernelized,csaji2019distribution}.
 \clearpage

\section{Details from \Cref{sec:preliminaries}}
\label{sec:app:preliminaries}
\paragraph{Alignment achieved by Gradient Descent}
In \Cref{sec:preliminaries} we discussed that any optimization algorithm that converges to the stationary point satisfies the alignment condition \Cref{lem:alignment}.
In this section, as an example, we demonstrate this for the \ac{GD} algorithm.

In particular, given the step size $\eta > 0$ and initial
parameter vector $\th_0 \in \R^p$ such that $L_{\lambda}(\th_0) < \infty$,
the parameter vector $\hat \th$ is obtained as the limit of a
sequence $\th_1, \th_2, \ldots$ defined through the standard \ac{GD} update rule
\begin{align*}
  \th_{t+1} = (1-\lambda \eta) \th_t - \eta \nabla L(\th_t) \qquad \text{for} \quad t=1,2,\ldots~.
\end{align*}
It is well known that the sequence converges under mild conditions, and so $\hat \th$ is a \emph{stationary point}.
This is implied by the well-known descent lemma:
\begin{lemma}[Descent lemma]
  Assuming that $\nabla L_\lambda$ is $\beta$-Lipschitz w.r.t.\ $L2$ norm
  and that $\eta$ satisfies $\eta \leq 1/\beta$, we have
\begin{align*}
  \frac{1}{2\beta} \|\nabla L_\lambda(\theta_t) \|^2 \le L_\lambda(\theta_t) - L_\lambda(\theta_{t+1}) \;.
\end{align*}
\end{lemma}
In particular, $\hat \th$ in this case satisfies the \emph{alignment} property (asymptotically aligns with the gradient):
\begin{lemma}[Alignment of GD]
Assume $L_\lambda$ is differentiable, $\nabla L_{\lambda}$ is $\beta$-Lipschitz w.r.t.\ $L2$ norm, and that the step size satisfies $\eta \leq 1/\beta$. Then, the limiting point $\hat \th = \th_{\infty}$ satisfies
$
\lambda \hat \theta = -\nabla L(\hat \theta)
$.
\end{lemma}
\begin{proof}
  Since $L_\lambda$ is differentiable and $\beta$-smooth with $\beta>0$, by the Descent Lemma,
summing over all iterates and using the fact that losses are non-negative,
\[
\frac{1}{2\beta} \sum_{t=0}^\infty \|\nabla L_\lambda(\th_t) \|^2 \le L_\lambda(\th_0) \;.
\]
Boundedness of the above sum implies that $\|\nabla L_\lambda(\th_t)\| \rightarrow 0$ as $t\rightarrow \infty$. Therefore, at the limit, we have the alignment of the weight and the gradient vectors, $\lambda \th = -\nabla L(\th)$. 
\end{proof}

\clearpage

\section{Details from \Cref{sec:results}}
\label{sec:app:results}
\subsection{Linear least squares}
\label{sec:linear-least-squares}
In case of ridge regression
\begin{align*}
  M(\hat \th)
  &= \pr{\frac1n \sumin x_i x_i\tp + \lambda I}^{-1} \pr{\frac1n \sumin x_i x_i\tp} \pr{\frac1n \sumin x_i x_i\tp + \lambda I}^{-1}\\
  &=
  \sigma^2 \pr{\frac1n \sumin x_i x_i\tp + \lambda I}^{-1} - \sigma^2 \lambda \pr{\frac1n \sumin x_i x_i\tp + \lambda I}^{-2}
\end{align*}
and so $\E\br{\|\nabla_{\hat \th} f(\xtest; \hat \th)\|_{M(\hat \th)}^2} \leq \|\xtest\|_{\Sigma^{-1}}^2$.
Now,
\begin{align*}
  \E[f(\xtest)]
  &=
  \xtest\tp \E[\hat \th]\\
  &= \xtest\tp \pr{\frac1n \sumin x_i x_i\tp + \lambda I}^{-1} \frac1n \sumin x_i \E[Y_i]\\
  &=
  \xtest\tp \pr{\frac1n \sumin x_i x_i\tp + \lambda I}^{-1} \frac1n \sumin x_i x_i\tp \th^{\star}\\
&=
  \xtest\tp \th^{\star} - \lambda \xtest\tp \pr{\frac1n \sumin x_i x_i\tp + \lambda I}^{-1} \th^{\star}~.
\end{align*}
Finally, note that
\begin{align*}
  \Var(\|\nabla_{\th} f(\xtest; \hat \th)\|_{M(\hat \th)}^2) \leq \E\br{\|\nabla_{\th} f(\xtest; \hat \th)\|_{M(\hat \th)}^4} -
  \E^2\br{\|\nabla_{\th} f(\xtest; \hat \th)\|_{M(\hat \th)}^2} = 0
\end{align*}
and applying \Cref{thm:ci-hp-ls} in this case, we have $c=0$ and $v = 0$.

\subsection{Relationship to Fisher information matrix and asymptotics}
Throughout this section, for symmetric matrices $A,B$ we will use $A\preceq B \iff B - A \succeq 0
$ to indicate that $\forall x ~:\;x^\top(B - A)x\ge0$.
First, we will need the following technical lemma.
\begin{lemma}
  \label{lem:E-of-inverse}
  Let $A \in \R^{p \times p}$ be an invertible matrix and let $X = \frac1n \sumin Z_i B_i$ be such that
  $Z_1, \ldots, Z_n \sim \cN(0, \sigma^2)$ and
  $B_1, \ldots, B_n \in \R^{p \times p}$ are fixed symmetric matrices with $\|B_i\|_{\mathrm{op}} \leq C$.
  Then, for any fixed $\alpha > 0$ having setting $\hat \lambda = \alpha + |\min\{0, \lambda_{\min}(X)\}|$, we have
  \begin{align*}
    \E[(A + X + \hat \lambda I)^{-1}] \preceq
    (A + \alpha I - \ve_n I)^{-1}
    +
    \frac{\sigma^2}{n} \, I
  \end{align*}
  and where
  \begin{align*}    
    \ve_n = 2 \sigma  \sqrt{\frac{C \ln\left(\frac{n \sqrt{2 p}}{\alpha}\right)}{n}}~.
  \end{align*}  
\end{lemma}
\begin{proof}
  We will use a concentration inequality spectral norms of of Gaussian series, which says that
\begin{theorem}[{\citet[Theorem 4.6.1]{tropp2015introduction}}]
  \label{eq:tail-bound-gaussian-series}  
Under conditions of the lemma on $Z_i, B_i$, for any $t > 0$,
\begin{align*}
  \P\cbr{\left\| \frac1n \sumin Z_i B_i \right\|_{\mathrm{op}} \geq t} \leq p \exp\pr{- \frac{n t^2}{2 \sigma^2 C}}~.
\end{align*}
\end{theorem}

  Define event $\cE = \{\|X\|_{\op} \leq t\}$ for some $t > 0$ to be determined later.
  Then,
  \begin{align*}
    \E[(A_{\lambda} + X)^{-1}]
    &=
    \E[(A_{\lambda} + X)^{-1} \ind_{\cE}]
    +
      \E[(A_{\lambda} + X)^{-1} \ind_{\neg \cE}]\\
    &\preceq
      \E[(A_{\lambda} - t I)^{-1} \ind_{\cE}]
    +
      \sqrt{\E[\|(A_{\lambda} + X)^{-1}\|_{\op}^2]} \sqrt{\P\cbr{\neg \cE}} \, I\\
    &\preceq
      \E[(A_{\lambda} - t I)^{-1}]
      +
      \sqrt{\E\br{\frac{1}{|\lambda + \lambda_{\min}(X)|^2}}} \, \sqrt{2 p} \exp\pr{- \frac{n t^2}{4 \sigma^2 C}} \, I
  \end{align*}
where we used a Cauchy-Schwartz inequality and a tail bound in \cref{eq:tail-bound-gaussian-series}.
Now we will set $\lambda$ to depend on $\lambda_{\min}(X)$ to keep the inverse well-defined in a deterministic sense.
  Set $\lambda = \hat \lambda = \alpha + |\min\{0, \lambda_{\min}(X)\}|$ where $\alpha > 0$ is a free variable to be set later.
  In this case
  \begin{align*}
    \E[(A_{\hat \lambda} + X)^{-1}]
    \preceq
      \E[(A_{\hat \lambda} - t I)^{-1}]
      +
      \frac{1}{\alpha} \, \sqrt{2 p} \exp\pr{- \frac{n t^2}{4 \sigma^2 C}} \, I
  \end{align*}

Putting all together and choosing
  \begin{align*}
    t = 2 \sigma  \sqrt{\frac{C \ln\left(\frac{n \sqrt{2 p}}{\alpha}\right)}{n}}
  \end{align*}
  we get the statement.
\end{proof}
An asymptotic bound presented in \Cref{para:fisher} is obtained from the following \Cref{prop:almost-fisher} assuming boundedness of $\|\nabla f\|, \|\nabla^2 f\|_{\op}$ at $(\th^{\star}, \xtest)$.
\begin{proposition}
  \label{prop:almost-fisher}
  Assume that $\max_i \|\nabla_{\th}^2 f(x_i; \th^{\star})\|_{\op} \leq C$.
  Then, there exists a data-dependent tuning of regularization parameter $\hat \lambda$ such that
  \begin{align*}
    \E[M(\th^{\star})]
    \preceq
    \cI(\th^{\star})^{-1}
+
    \frac{\sigma^2}{n} \, I~.
  \end{align*}
\end{proposition}
\begin{proof}
  Note that
  \begin{align*}
    \nabla^2 L(\th^{\star}) &= \cI(\th^{\star}) + X~, \qquad \text{where}\\
    \cI(\th^{\star}) &= \frac1n \sumin \nabla_{\th} f(x_i; \th^{\star}) \nabla_{\th} f(x_i; \th^{\star})\tp~,
                       \qquad X = \frac1n \sumin Z_i \nabla_{\th}^2 f(x_i; \th^{\star})
  \end{align*}
  and so
  \begin{align*}
    \E[M(\th^{\star})]
    &= \E\br{(\nabla^2 L(\th^{\star}) + \lambda I)^{-1} \cI(\th^{\star}) (\nabla^2 L(\th^{\star}) + \lambda I)^{-1}}\\
    &= \E\br{(\nabla^2 L(\th^{\star}) + \lambda I)^{-1}}
    - \E\br{(\nabla^2 L(\th^{\star}) + \lambda I)^{-1} \pr{X + \lambda I}
      (\nabla^2 L(\th^{\star}) + \lambda I)^{-1}
        }\\
    &\preceq
      \E\br{(\nabla^2 L(\th^{\star}) + \hat \lambda I)^{-1}}
  \end{align*}
where we set $\lambda = \hat{\lambda} = \alpha + |\min\{0, \lmin(X)\}|$ for a free parameter $\alpha > 0$ which ensures that $X + \lambda I $ is \ac{PSD}.
Finally applying \Cref{lem:E-of-inverse} we get
  \begin{align*}
    \E[M(\th^{\star})]
    \preceq
    (\cI(\th^{\star}) + \alpha I - \ve_n I)^{-1}
+
    \frac{\sigma^2}{n} \, I
  \end{align*}
  where $\ve_n$ is given in \Cref{lem:E-of-inverse}.
  But now let us tune $\alpha$ to satisfy
  \begin{align*}
    \alpha > \ve_n = 2 \sigma  \sqrt{\frac{C \ln\left(\frac{n \sqrt{2 p}}{\alpha}\right)}{n}}~.
  \end{align*}
  The solution to the above inequality involves the Lambert function and further upper-bounding it using~\citep[Theorem 2.3 with $y=1$]{HoorfarH08} we eventually have
  \begin{align*}
    \alpha > c \sqrt{\ln(b/\alpha)} \quad \implies \quad \alpha >  \sqrt{\frac{b^2}{2 b^2 / c^2 + 1}}
    \qquad (b,c > 0)~.
  \end{align*}
  Thus, we have
  \begin{align*}
    \alpha > \sqrt{\frac{2 n^2 p}{\frac{n^3 p}{C \sigma ^2}+1}}
  \end{align*}
  and so the complete setting of $\lambda$  that gives the statement is
  \begin{align*}
    \hat{\lambda} = \lambda = \sqrt{\frac{2 n^2 p}{\frac{n^3 p}{C \sigma ^2}+1}} + |\min\{0, \lmin(X)\}| > 0~.
  \end{align*}

\end{proof}
\begin{proposition}
  \label{prop:asymptotics}
  Assume that $\th \mapsto f(\xtest; \th)$ is continuously differentiable, and that $L_{\lambda}$ is twice continuously differentiable.
  Assume that the sequence $(\E[\|\nabla_{\th} f(\xtest; \hat \th_{n})\|^2_{M(\hat \th_n)}])_{n=1}^{\infty}$ converges.
  Then, almost surely
\begin{align*}  
  \limsup_{n \to \infty} \frac{\sqrt{n}}{\ln(n)} \, \abs{f(\xtest; \hat \th_n) - \E[f(\xtest; \hat \th_n)]}
  \leq
  \sigma \pi \sqrt{\E \|\nabla_{\th} f(\xtest; \hat \th_{\infty})\|_{M(\hat \th_{\infty})}^2}
\end{align*}
\end{proposition}
\begin{proof}
Setting $\delta = 2/n^2$ in \Cref{thm:ci-hp-ls} we have
\begin{align*}
  \sum_{n=1}^{\infty} \P\cbr{
  \sqrt{\frac{n}{2 \sigma^2 \ln(n)}}
  |\Delta(\xtest; \hat \th_n)|
  >
  \sqrt{\frac{\pi^2}{2} \, \E[\|\nabla_{\th} f(\xtest; \hat \th_{n})\|^2_{M(\hat \th_n)}]}
  +
  c_3 \, \sqrt{\frac{2 \ln(n)}{n}}
  }
  \leq
  \frac{\pi^2}{3}
\end{align*}
which by the Borel-Cantelli lemma gives
\begin{align*}
  \P\cbr{
  \limsup_{n \to \infty}
  \sqrt{\frac{n}{2 \sigma^2 \ln(n)}} \,
  |\Delta(\xtest; \hat \th_n)|
  >
  \sqrt{\frac{\pi^2}{2} \, \lim_{n \to \infty} \E[\|\nabla_{\th} f(\xtest; \hat \th_{n})\|^2_{M(\hat \th_n)}]}
  }
  = 0
\end{align*}
where we assumed that the sequence $(\E[\|\nabla_{\th} f(\xtest; \hat \th_{n})\|^2_{M(\hat \th_n)}])_{n=1}^{\infty}$ converges.
Finally, invoking DCT, and assuming that $f$ is continuously differentiable, and that $L_{\lambda}$ is twice continuously differentiable,
\begin{align*}
  \lim_{n \to \infty} \E[\|\nabla_{\th} f(\xtest; \hat \th_{n})\|^2_{M(\hat \th_n)}]
  =
  \E[\|\nabla_{\th} f(\xtest; \hat \th_{\infty})\|^2_{M(\hat \th_\infty)}]~.
\end{align*}
\end{proof}

\subsection{Uniform bound over test inputs on the unit sphere}
\label{sec:uniform-bound}
Recall that $\texttt{bound}(\ln(1/\delta))$ is the bound on $\Delta(\xtest; \hat \th)$ given in \Cref{thm:ci-hp-ls}, as a function of $\ln(1/\delta)$ term.
Then,
\begin{align*}
  \P\cbr{
  \max_{x \in \mathbb{S}^{d-1}} |\Delta(x; \hat \th)| >
  \texttt{bound}\pr{d \, \ln\pr{\frac{3 n \Lip(\Delta)}{\delta}}}
  +
  \frac1n
  } \leq \delta
\end{align*}
where $\Lip(\Delta)$ is a Lipschitz constant of the gap $x \mapsto \Delta(x; \hat \th)$.
\begin{proof}
  Let $\cC_{\ve}$ be an $\ve$-cover of $\mathbb{S}^{d-1}$.
It is well-known that $|\cC_{\ve}| \leq (3/\ve)^d$.
Then, for all $\ve > 0$,
\begin{align*}
  \max_{x \in \mathbb{S}^{d-1}} \Delta(x)
  &=
    \max_{x \in \mathbb{S}^{d-1}} \min_{y \in \cC_{\ve}}\cbr{ \Delta(x) - \Delta(y) + \Delta(y) }\\
  &\leq
    \max_{x \in \mathbb{S}^{d-1}} \min_{y \in \cC_{\ve}} \Lip(\Delta) \|x - y\|
    +
    \max_{y \in \cC_{\ve}} \Delta(y)\\
  &\leq
    \cbr{ \Lip(\Delta) \, \ve + \texttt{bound}\pr{\ln\pr{\frac{|\cC_{\ve}|}{\delta}}} }\\
  &\leq
    \cbr{\Lip(\Delta) \, \ve + \texttt{bound}\pr{d \, \ln\pr{\frac{3}{\ve \delta}}} }\\
  &\leq
    \frac1n + \texttt{bound}\pr{d \, \ln\pr{\frac{3 n \Lip(\Delta)}{\delta}}}~.
\end{align*}
where we (lazily) set $\ve = 1/(n \Lip(\Delta))$.
\end{proof}
The following is required for our example of a uniform bound when the predictor is a multilayer neural network.
Eventually, the proof involving a covering number argument is given at the end of this section.
\paragraph{Lipschitz constant of a neural network}
Neural networks defined in \cref{eq:fcn} can conveniently be written as
\begin{align}
  \label{eq:nn-as-prod}
  h_k = D_k W_k D_{k-1} W_{k-1} \cdots D_1 W_1 x,
\end{align}
where $D_k$ is a diagonal matrix defined as $D_k = \diag(\sigma'(W_k h_{k-1}))$, and both $(D, h)$ depend implicitly on the input $x$.  
Then, the derivative with respect to a weight matrix has a form:
\begin{align*}
  \frac{\diff f(x; \th)}{\diff W_k}
  &=
    (W_K D_{K-1} \cdots W_{k+1} D_k)\tp (h_{k-1})\tp~.
\end{align*}
Now we give a bound on the Lipschitz constant of a neural network.
\begin{lemma}
  \label{lem:fcn-lip}
  Assume that $f(x; \hat \th)$ is a $K$-layer neural network as in
  \cref{eq:fcn} with 
activation $a : \R \to \R$ such that
  $\|a'\|_{\infty} \leq 1$.
  Then, for any parameter vector $\th$,
  \begin{align*}
    \Lip(f(\cdot \,; \th))
    \leq
    \pr{\frac{L_{\lambda}(\th)}{\lambda K}}^{\frac{K}{2}}~.
  \end{align*}
  Moreover, if $\th_t$ is a parameter vector found by \ac{GD}
  \begin{align*}
    \Lip(f(\cdot \,; \th_t))
    \leq
    \pr{\frac{L_{\lambda}(\th_0)}{\lambda K}}^{\frac{K}{2}}~.
  \end{align*}
\end{lemma}
\begin{proof}
Focusing on the Lipschitz constant of a neural network $f$, we first compute
\begin{align*}
  \frac{\d h_k}{\d x}
  =
  a' (W_k h_{k-1}) W_k \frac{\d h_{k-1}}{\d x}
  =
  \cdots
  =
  \prod_{i=1}^k a' (W_i h_{i-1}) W_i~.
\end{align*}
Then, assuming that $a'$ is uniformly bounded by $1$,
\begin{align*}
  \left\|\frac{\d h_K}{\d x}\right\|
  \leq
  \prod_{i=1}^K \|W_k\|_{\mathrm{op}}~.
\end{align*}
Therefore,
 \begin{align*}
   \Lip(f(\cdot \,; \th))
   &\le \prod_{k=1}^K \|W_k \|_{\mathrm{op}}\\
     &= \left( \prod_{k=1}^K \|W_k \|_{\mathrm{op}}^2 \right)^{1/2} \\
    &\le \left( \frac{1}{K}\sum \|W_k \|_F^2 \right)^{K/2} = \left( \frac{\|\th\|_2^2}{K} \right)^{K/2} \tag{AM-GM inequality}\\
    &\le \left( \frac{L_\lambda(\th)}{\lambda K} \right)^{K/2} \;.
 \end{align*}

 If $\th_t$ is found by \ac{GD}, $t \mapsto L_{\lambda}(\th_t)$ is non-increasing (as one can see from the descent lemma),
 for all $t \geq 0$, and the second result follows.
\end{proof}
\paragraph{Uniform bounds through covering argument.}
Finally we control the Lipschitz constant of $\Delta$ as
  \begin{align*}
    \Lip(\Delta)
    &= \sup_{x \in \mathbb{S}^{d-1}} \|\nabla_x f(x; \hat \th) - \E \nabla_x f(x; \hat \th)\|\\
    &\leq \sup_{x \in \mathbb{S}^{d-1}}\|\nabla_x f(x; \hat \th)\| + \sup_{x \in \mathbb{S}^{d-1}}\|\nabla_x \E[f(x; \hat \th)]\|\\
    &= \Lip(f(x; \hat \th)) + \Lip(\E[f(x; \hat \th)])\\
    &\leq \pr{\frac{L_{\lambda}(\th_0)}{\lambda K}}^{\frac{K}{2}} + \E \pr{\frac{L_{\lambda}(\th_0)}{\lambda K}}^{\frac{K}{2}}\\
    &\leq 2 \pr{\frac{C}{\lambda K}}^{\frac{K}{2}}
  \end{align*}
  where the last inequality comes by \Cref{lem:fcn-lip}.
\clearpage
  \subsection{Algorithm pseudocode from \Cref{sec:alg}}
  \label{sec:app:alg}
\begin{algorithm}
\caption{Efficient computation of the weighted norm}
\label{alg:weighted-norm}
\begin{algorithmic}[1]
\Require Loss $L_{\lambda}$, parameter vector $\hat \th \in \R^p$, inputs $x_1, \ldots, x_n, \xtest \in \mathbb{R}^p$, error tolerance $\ve > 0$.
\Ensure Approximation of a weighted norm $\hat V \approx \|\nabla_\th f(\xtest; \hat \th)\|_{M(\hat \th)}^2$.

\State Compute $h \gets (\nabla^2 L(\hat \th) + \lambda I)^{-1} \nabla_\th f(\xtest; \hat \th)$ by running \Cref{alg:hinvp} with $v = \nabla_\th f(\xtest; \hat \th)$, intialization $h_0 = 0$, and error tolerance $\ve$.
\If{\textit{interpolation}}
\State Compute
$
  \hat V \gets \nabla f(x; \hat \th)\tp h
$
\Comment{$\cO(p)$}
\Else 
\State Compute
$
  \hat V \gets \frac1n \sumin (\nabla f(x_i; \hat \th)\tp h)^2
$
\Comment{$\cO(p n)$}
\EndIf
\State \Return $\hat V$
\end{algorithmic}
\end{algorithm}
\begin{algorithm}
\caption{Efficient Hessian-Inverse-Vector Product computation using autodiff and \ac{CG}}
\label{alg:hinvp}
\begin{algorithmic}[1]
\Require Loss $L_{\lambda}$, parameters $\hat \th \in \R^p$, target vector $v \in \mathbb{R}^p$, initialization $h_0 \in \R^p$, error tol. $\ve > 0$.
\Ensure Approximation $h \approx [\nabla^2 L_{\lambda}(\hat \th)]^{-1} v$.

\Function{HessianVectorProduct}{$L, \theta, z$}
\State {\color{blue} \# Computes the product $Hz$ without forming $H$}
    \State Compute gradient $g(\th) = \nabla L(\theta)$ \Comment{$\cO(C_L)$}
    \State Compute scalar $s(\th) = g(\th)\tp z$ \Comment{$\cO(p)$}
    \State Compute $Hz = \nabla s(\th) = \nabla_\th[ (\nabla L(\theta))\tp z]$ \Comment{Using autodiff $\cO(C_L)$, not $\cO(p^2)$}
    \State \Return $Hz$
\EndFunction
\vspace{0.5em} 

\State {\color{blue} \# Define the matrix-vector product operation $A(z)$ required by \ac{CG}:}
\State $A(z) := \Call{HessianVectorProduct}{L, \theta, z}$
\vspace{0.5em}

\State {\color{blue} \# Solve the linear system $H h = v$ for $h$ using \ac{CG}:}
\State $h \leftarrow \Call{ConjugateGradient}{A, v, h_0}$
\Comment{Requires $k = \Omega\pr{\sqrt{1 + \frac{1}{\lambda}} \ln(\frac{1}{\ve})}$ iterations}
\State \Return $h$

\end{algorithmic}
\end{algorithm}   
\clearpage
\section{Other omitted proofs}
\label{sec:omitted-proofs}
\subsection{Pointwise bound with polynomial dependence on $\delta$}
\label{sec:poly-rates}
\begin{theorem}
  Assume that $\hat \th$ is a local minimizer of $L_{\lambda}$ and that the test input $\xtest$ is fixed.
  Then, for any $\delta \in (0,1)$
  \begin{align*}
    \P\cbr{
    |f(\xtest; \hat \th) - \E[f(\xtest; \hat \th)]|
    >
    \sigma \sqrt{ \frac{1}{\delta n}
    \E\big[ \|\nabla_\th f(\xtest; \hat \th)\|_{M(\hat \th)}^2 \big]
    }
    } \leq \delta  
  \end{align*}
\end{theorem}
To prove this theorem, we start with a basic
concentration inequality which tells us the function of multiple random elements concentrates well around its mean when its conditional variances are small:
\begin{proposition}
  \label{prop:cheb}
  Let $S=(Z_1, \ldots, Z_n)$ be a tuple of independent random elements and let $S'$ be its independent copy.
  Then, for continuously differentiable $g : \cZ^n \to \R$ and for any $\delta \in (0, 1]$,
  \begin{align*}
    \P\cbr{|g(S) - \E[g(S)]| > \sqrt{\frac{1}{\delta} \sumin \E \Var(g(S) \mid S\deli)}} \leq \delta~.
  \end{align*}
\end{proposition}
\begin{proof}
  For i.i.d. random variables $X,Y$ we have $\Var(X) = \frac12 \E[(X-Y)^2]$.
  Thus,
  \begin{align*}
    \Var(g(S) \mid S\deli)
    =
    \E\br{(g(S) - \E[g(S) \mid S\deli])^2 \bmid S\deli}
    =
    \frac12 \E\br{\pr{g(S) - g(S\repi)}^2 \bmid S\deli}
  \end{align*}
  that is where we took expectation only with respect to $Z_i, Z_i'$.
  On the other hand, by the Efron-Stein inequality~\cite{boucheron2013concentration}
  \begin{align*}
    \Var(g(S)) \leq \frac12 \sumin \E[(g(S) - g(S\repi))^2]
    = \sumin \E \Var(g(S) \mid S\deli)~.
  \end{align*}
  Finally, the application of Chebyshev's inequality completes the proof.
\end{proof}
We first make a connection between conditional variances and gradients through the following key general inequality~\citep{boucheron2013concentration}:
\begin{lemma}[Gaussian Poincar\'e inequality]
  \label{lem:poincare}
  Let $Z$ be a Gaussian random vector, distributed according to $\cN(0, I_d)$.
  Let $f : \R^d \to \R$ be any continuously differentiable function.
  Then,
  \begin{align*}
    \Var(f(Z)) \leq \E[\|\nabla f(Z)\|^2]~.
  \end{align*}
\end{lemma}
Combining this lemma with \Cref{prop:cheb} already establishes that concentration is controlled by the squares of the gradient norms.
All that is left to do is to analyze them for $f(\xtest; \hat \th)$, which is done in the following lemma.
Putting together \Cref{prop:cheb,lem:poincare,lem:norm2norm}
completes the proof.

\end{document}